%% file: ewrl_2026/ewrl_2026_main.tex
\documentclass{article}

\PassOptionsToPackage{numbers,sort&compress}{natbib}

\usepackage[final]{ewrl_2026}

\usepackage[utf8]{inputenc} 
\usepackage[T1]{fontenc}    
\usepackage{hyperref}       
\usepackage{url}            
\usepackage{booktabs}       
\usepackage{amsfonts}       
\usepackage{nicefrac}       
\usepackage{microtype}      
\usepackage{xcolor}         

\usepackage{amsmath}
\usepackage{amssymb}
\usepackage{mathtools}
\usepackage{amsthm}
\usepackage{bbm}
\usepackage{url}
\usepackage{hyperref}

\usepackage{dblfloatfix}

\usepackage{todonotes}

\usepackage{algorithm}
\usepackage{algpseudocode}

\usepackage{etoc}

\newtheorem{theorem}{Theorem}

\usepackage{listings}

\input{math_commands.tex}

\def\A{\mathcal{A}}

\def\O{\mathcal{O}}
\def\I{\mathcal I}

\def\A{\mathcal A}
\def\E{\mathbb E}
\def\V{\mathbb V}

\title{Particle Monte Carlo Tree Search}

\author{%
  Yaniv Oren\thanks{Equal contribution.} \\
  Department of Intelligent Systems\\
  Delft University of Technology\\
  2628 CD Delft, The Netherlands \\
  \texttt{y.oren@tudelft.nl} \\
  \And
  Viliam Vadocz\footnotemark[1] \\
  Department of Computer Science \\
  ETH Zürich \\
  8092 Zürich, Switzerland \\
  \texttt{vvadocz@student.ethz.ch} \\
  \And
  Joery A. de Vries \\
  Trent AI Limited \\
  W1W 7LT London, England \\
  \texttt{joery@trent.ai} \\
  \And 
  Wendelin B{\"o}hmer \\
  Department of Intelligent Systems\\
  Delft University of Technology\\
  2628 CD Delft, The Netherlands \\
  \texttt{j.w.bohmer@tudelft.nl} \\
  \And 
  Matthijs T. J. Spaan \\
  Department of Intelligent Systems\\
  Delft University of Technology\\
  2628 CD Delft, The Netherlands \\
  \texttt{m.t.j.spaan@tudelft.nl} \\
  \And 
  Hendrik Baier \\
  Information Systems\\
  TU Eindhoven\\
  5612 AE Eindhoven, The Netherlands \\
  \texttt{h.j.s.baier@tue.nl}
}

\begin{document}
\etocdepthtag.toc{main}

\maketitle

\begin{abstract}
\input{abstract}
\end{abstract}

\input{sections/introduction}
\input{sections/background}
\input{sections/particle_mcts}
\input{sections/theoretical_analysis}
\input{sections/related_work}
\input{sections/experiments}
\input{sections/conclusions}

\begin{ack}
This work used the Dutch national e-infrastructure with the support of the SURF Cooperative using grant no. EINF-13454.
Research reported in this work was partially or completely facilitated by computational resources and support of the Delft AI Cluster (DAIC) at TU Delft (RRID: SCR\_025091), but remains the sole responsibility of the authors, not the DAIC team.
\end{ack}

\bibliographystyle{unsrtnat}
\bibliography{bib}


\appendix
\etocdepthtag.toc{appendix}

\newpage

{
  \etocsettagdepth{main}{none}
  \etocsettagdepth{appendix}{subsection}
  \etocsettocstyle{\section*{Appendix Contents}}{}
  \tableofcontents
}

\newpage

\input{appendix/symbols}

\input{appendix/pseudocode}

\input{appendix/derivations}

\input{appendix/proofs}

\input{appendix/additional_prior_work}

\input{appendix/additional_results}

\input{appendix/implementation_details}

\input{appendix/experimental_details}

\input{appendix/hyperparameters}


\end{document}

%% file: math_commands.tex
\usepackage{amsmath,amsfonts,bm}

\def\eqref#1{equation~\ref{#1}}

\def\1{\bm{1}}

\DeclareMathAlphabet{\mathsfit}{\encodingdefault}{\sfdefault}{m}{sl}
\SetMathAlphabet{\mathsfit}{bold}{\encodingdefault}{\sfdefault}{bx}{n}

\newcommand{\E}{\mathbb{E}}

\DeclareMathOperator*{\argmax}{arg\,max}


%% file: abstract.tex
Monte Carlo Tree Search (MCTS) is a widely used approach for policy improvement and action selection in Reinforcement Learning.
Due to its sequential and deterministic nature, principled runtime-scaling of MCTS with parallel compute remains a major challenge.
We introduce Particle MCTS (PMCTS), a parallel MCTS algorithm which is suited for neural network evaluations, designed for GPU-acceleration with batch-parallelization and retains MCTS's principled approximate policy improvement interpretation.
Empirically, PMCTS scales well with parallel compute and consistently outperforms or compares well to the popular heuristic-based baselines across a range of popular discrete- and continuous-action benchmark domains, including Chess, 19x19 Go, 9x9 Go, Gardner Chess, Snake, classical control environments from Brax and LLM reasoning in Sokoban.

%% file: sections/introduction.tex
\section{Introduction}
\label{sec:intro}

Monte Carlo Tree Search \citep[MCTS,][]{UCT} is a popular search algorithm that plays a central role in modern reinforcement learning (RL).
MCTS performs \textit{policy improvement} and action selection through lookahead and evaluation (i.e., \textit{search}) in a model of the dynamics of the environment.
This paradigm has been popularized in RL by the AlphaZero family of algorithms \citep{AlphaGo,AlphaZero} and has since been successfully applied to a wide range of domains, including zero-sum games \citep{AlphaZero}, discrete-action, single-agent tasks \citep{muzero,AlphaTensor,AlphaDev}, continuous control \citep{effzero_v2} and LLM inference scaling \citep{hao2023reasoning,
feng2024alphazero,
shi2025monte}.
This wide applicability of MCTS makes its inference-time scaling with parallel compute a critical challenge and the subject of significant attention \citep{parallel_uct,virtual_loss,parallel_mcts_2,on_effective_parallelization_of_mcts,massively_parallel_mcts_iclr,batch_mcts}.

Parallelizing MCTS is challenging for two main reasons: 
(I)
In each iteration, MCTS selects a trajectory for evaluation using conventionally deterministic selection policies that depend on statistics from previous iterations in a deterministic dynamics model \citep{UCT,puct,gumbel_mz}.
This makes the \textit{principled} selection of multiple \textit{unique} trajectories in parallel a significant challenge, because parallel selection steps result in the same trajectory.
(II) Modern MCTS methods use a deterministic deep neural network (DNN) for evaluation.
This makes scaling the evaluation of a \textit{single} trajectory with parallel compute a significant challenge, because repeated evaluations of the same trajectory result in duplicate evaluations rather fresh evaluations which average out error.
Together, these properties also make independent instances of MCTS at the same state (\textit{root parallelization}) expand similar trees and return similar search results, renders root parallelization as ineffective for inference-time scaling and makes the selection of multiple unique trajectories in parallel an important challenge.

To overcome these challenges popular methods use heuristic selection policies such as \textit{Virtual Losses} \citep{virtual_loss} and \textit{Virtual Means} \citep{batch_mcts} which rely on communication between processes during selection.
These methods suffer from three main limitations: 
First, their heuristic selection policy does not have a clear, principled interpretation as an improved policy or as tracking an improved policy.
Second, communication during selection hamstrings parallelization and precludes batch-parallelization of the complete search.
Batched-parallelization is an effective and popular approach for GPU-acceleration in RL specifically and Deep Learning in general \citep{jax2018github}.
Third, they do not account for the effect of duplicate trajectories on the evaluation.

To address these limitations, in this work we introduce \emph{Particle MCTS} (PMCTS), a novel parallel MCTS algorithm which retains modern MCTS's principled approximate policy improvement interpretation \citep{grill2020monte}, is well suited for GPU acceleration through batch-parallelization and explicitly accounts for the presence of duplicate trajectories.
PMCTS samples multiple trajectories independently in parallel from a theoretically motivated stochastic selection policy; 
explicitly controls the entropy of the selection policy to increase trajectory diversity;
uses Sequential Monte Carlo \citep[SMC, see][]{chopin2020introduction} within the framework of MCTS to maintain principled estimates of the expected return with respect to the intended target policy and account for the presence of remaining duplicate trajectories;
uses the Effective Sample Size \citep[ESS, see ][]{chopin2020introduction} to estimate the variance of the SMC expectation estimator and maintain principled updates to the search tree;
and last, uses importance sampling to mitigate the effect of catastrophic returns observed during evaluation of new nodes in the tree.

We show that Simple PMCTS and PMCTS retain MCTS's principled approximate policy-improvement interpretation and characterize the manner in which increased parallelization improves evaluation.
Empirically, PMCTS demonstrates strong inference-time scaling, outperforming or comparing well to the popular, heuristic-based baselines across a range of popular benchmark domains, including hard zero-sum games (chess, 19x19 Go, 9x9 Go, Gardner chess), discrete-action, single-agent tasks (Snake and RubiksCube \citep{bonnet2024jumanji}), classical continuous control (Ant, HalfCheetah and Humanoid \citep{brax}) and an LLM inference scaling task, scaling the capacity of the LLM to play the game of Sokoban \citep{bonnet2024jumanji}.
PMCTS provides a promising new direction for scaling runtime inference with MCTS and parallel compute, which is principled for policy improvement and action selection and well suited for modern GPU acceleration.

%% file: sections/background.tex
\section{Background}
\label{sec:background}
In RL, the environment is represented by a Markov Decision Process \citep[MDP,][]{bellman1957markovian} $ \mathcal{M} = \langle \mathcal{S}, \mathcal{A}, \rho, R, P,
\gamma \rangle $. 
$ \mathcal{S} $ is a set of states, $ \mathcal{A}$ is a set of actions, both of which we consider finite for presentation and analysis simplicity.
$ \rho $ is an initial state distribution,
$ R: \mathcal{S} \times \mathcal{A} \to \mathbb{R} $ is a bounded possibly stochastic reward function and $ P $ is a transition distribution $s' \sim P(\cdot | s, a)$.
The policy of the agent $\pi \in \Pi$ is defined as a distribution over actions $a \sim \pi(s)$ with probabilities specified by $\pi(a|s)$. 
The objective $ J_\pi $ is the maximization of the expected discounted return, or state-value $V^\pi(s)$:
\begin{align}
    \label{eq:rl_objective}
    J_\pi = \E_{\rho} \big[V^\pi(s_1)\big] 
    = 
    \E\Big[
        {\textstyle \sum_{t=1}^{H}} \gamma^{t-1}R(s_t, a_t)
    \Big], 
    \quad s_1 \sim \rho, \, a_t \sim \pi(s_t), \, s_{t+1} \sim P(s_t, a_t),
\end{align}
over the initial state distribution $\rho$.
The discount factor $0 < \gamma < 1$ is used when $H\to\infty$ to guarantee that $V^\pi$ remains bounded, or for stability; otherwise, usually $\gamma = 1$.
A state-action \textit{Q-value function} is defined as follows:
$Q^\pi(s, a) = \E[R(s, a) + \gamma V^\pi(s')]$.

\subsection{Search for policy improvement}
RL algorithms which have access to a dynamics model of the environment $\mathcal{M} = (P, R)$ (exact or approximate) often use \textit{search} to drive the iterative improvement of the acting policy \citep{moerland2023model}.
Search refers to the process of look-ahead in the model from a state $s$, often using a DNN-based prior policy~$\pi_\theta$ and state value function $v_\phi$, in order to extract an improved policy $\pi_{\text{search}}(s)$ at the current state~$s$.
A popular example is AlphaZero which uses MCTS for that purpose.

For exact action values $Q^\pi$, a policy improvement operator $\I:\Pi\times\mathcal Q\to\Pi$ maps $\pi$ to a policy satisfying $V^{\I(\pi,Q^\pi)}(s)\geq V^\pi(s)$ at every state (i.e. policy improvement). 
We define $\mathcal Q$ as the set of bounded state-action functions, to indicate that $\I$ can also be applied to approximate search values $q\in\mathcal Q$; improvement in the search objective with such values does not by itself guarantee improvement in true MDP value however.

Search algorithms generally return: (I) A selected action $a$, (II) An improved policy $\pi_{\text{search}}(s)$ and (III) A value $V_{\text{search}}(s) \approx V^{\pi_{\text{search}}}(s) $.
$\pi_{\text{search}}(s)$ is used to train $\pi_\theta$ with cross entropy loss and $V_{\text{search}}(s)$ is used to produce bootstraps for TD-targets \citep{muzero} or even value targets directly \citep{oren2025epistemic}.

\subsection{Monte Carlo Tree Search (MCTS)}
The MCTS algorithm maintains a search tree over states with its root $s =: s_1$ representing the current state $s$ in the environment.
Traditionally, MCTS is used for search with an exact, deterministic model $\mathcal{M}$ and finite state-action environments.
For simplicity of presentation, we will maintain these assumptions hence.
However, MCTS has been used successfully for policy improvement and action selection with learned \citep{muzero} and stochastic transition \citep{stochastic_mz} models as well as in continuous state-action spaces~\citep{sampled_mz,effzero_v2}.
Our experiments also use sampled continuous actions, whereas the formal results in Section~4 assume finite actions and exact deterministic dynamics.
The tree is constructed by iterating (I) \textit{selection}, (II) \textit{expansion} and (III) \textit{backpropagation} $M$ times (often called \textit{number of expansions}, \textit{simulations} or the \textit{search budget}):

\textbf{Selection:} 
At iteration $i$, a trajectory $\tau^i_{T} = (s^i_1,a^i_1, \dots, s^i_T,a^i_T)$ in the tree is selected using a \textit{selection policy} $\pi_{i}$.
Traditionally, MCTS uses deterministic selection policies such as $\pi^{\text{UCT}}_{i}$ or $\pi^{\text{PUCT}}_{i}$ \citep[][see Appendix \ref{app:puct_and_uct} for additional detail]{UCT,puct,AlphaGo}.
$\pi_{i}$ uses the current state-action evaluation $q_i$, the number of prior backpropagation steps $M_i(s_t)$ through each node $s_t$ (and $M_i(s_t,a) $ which can be computed using $M_i(s_{t+1})$ of the node associated with the transition $(s_t,a)$) and often a \textit{prior policy} $\pi_\theta$.
\citet{grill2020monte} show that the visit distributions induced by UCT and PUCT track solutions of regularized policy-optimization problems defined by the search values $q_i$. 
They also evaluate selecting directly from the corresponding improved policy, obtaining better performance in their experiments.
A canonical modern version, GumbelMCTS \citep{gumbel_mz} uses the 
operator $\I_{\text{reg}}$:
\begin{align}
    \label{eq:pi_i}
    \pi_i(a|s_t) &= \I_{\text{reg}}(\pi_\theta,q_i)(a|s_t) 
    \propto \exp(\log \pi_\theta(a|s_t) + \beta_{i}(s_t) q_{i}(s_t,a)).
\end{align}
GumbelMCTS uses the heuristic $ \beta_i(s_t) = (c_{\text{visit}} + \max_{a \in \A} M_i(s_t, a))c_{\text{scale}} $, inspired by $\pi^{\text{PUCT}}_{i},$ which uses the number of visits $ M_i(s_t, a) $ as an estimate for the inverse-variance of the evaluation.
To compute $\pi_{i} $, an evaluation $q_{i}(s_t,a)$ is needed for \textit{all} actions, including those that have not yet been searched.
To achieve this, GumbelMCTS uses completed Q values, assigning $q_i(s_t,a)=v_{\mathrm{mix}}(s_t)$ when $M_i(s_t,a)=0$ \citep{gumbel_mz}. 
The mixed value interpolates $v_\phi(s_t)$ with the $\pi_\theta$-weighted mean of visited action values according to the total visit count, as defined in Appendix~\ref{app:action_selection_and_improvement_survey}.
In our experiments, we adhere to the design choices of GumbelMCTS whenever possible and unless specified otherwise.

\textbf{Expansion:} 
When selection arrives at an unobserved transition $(s_{T},a_{T})$ (a \textit{leaf}), it is \textit{expanded}. 
This gathers information from the model: $r_T = \E[R(s_{T},a_{T})], s_{T+1} \sim P(s_{T},a_{T}), v_{i+1}(s_{T+1}) \approx V^{\pi_\theta}(s_{T+1}) $.
In classical MCTS, the new node is evaluated using the average returns of $K$ rollouts in the model,
sampled i.i.d. from state $s_{T+1}$ 
using a \textit{rollout policy}.
Following AlphaGo \citep{AlphaGo}, modern variants use a DNN $v_{i+1}(s_{T+1}) := v_\phi(s_{T+1})$ instead.

\textbf{Backpropagation:} The bootstrapped return of the selection trajectory $\nu_i(s_t) = \sum_{j = 0}^{T-t}\gamma^j r_j + \gamma^{T-t+1} v_{i+1}(s_{T+1})$ is propagated backwards through all nodes $s_t$ along the trajectory. 
The nodes maintain the incoming reward $r(s_{t-1},a_{t-1})$ and the mean of returns from simulations that visited it:
\begin{align}
    v_{i+1}(s_t) = {\textstyle\frac{1}{M_{i+1}(s_t)} \sum_{j=1}^{M_{i+1}(s_t)}} \nu_j(s_t), 
    \,\,
    \,\,
    q_{i+1}(s_{t},a_{t}) = r(s_{t}, a_{t}) + \gamma v_{i+1}(s_{t+1}).
\end{align}
MCTS returns an improved policy, an action to take in the environment and additional statistics such as the value at the root and of actions at the root, as used by the learning algorithm.
The exact improved policy, action selection and value targets used by different MCTS variants, as well as PMCTS, are specified in Appendix~\ref{app:action_selection_and_improvement_survey}.

\subsection{Sequential Monte Carlo for search for policy improvement in RL}
\label{background:smc}
SMC methods use \textit{sequential importance sampling} (SIS) to approximate a sequence of \textit{target distributions} using \textit{proposal distributions} \citep{chopin2020introduction}.
In RL, SMC methods have been adapted for search for policy improvement \citep{piche2019} by formulating the proposal and target as distributions over trajectories~$\tau_t $ using $\pi_\theta$ as the \textit{proposal policy} and $\pi'(s_t) = \I(\pi_\theta, Q^{\pi_\theta})(s_t)$ as the \textit{target policy} \citep{tsmcts}.
At each time step $t \in \{1, \dots, T\}$, $N$ particles, each comprising a trajectory and a weight $\{\tau_t^n, w_t^n\}_{n=1}^N$, are updated via \textit{sampling} and \textit{correction}.
\textit{Sampling}: each trajectory $\tau^n_{t-1}$ is extended by sampling
$
    s_{t+1}^n \sim P(\cdot | s_{t}^n, a_{t}^n), 
    a_t^n \sim \pi_\theta(\cdot | s_t^n),
    s_1^n = s,
$
where $s$ is the current state in the environment.
\textit{Correction}: The weights are corrected using SIS (left) such that the weighted particles $\{\tau_t^n, w_t^n\}_{n=1}^N$ provide a consistent approximation of expectations under the target:
\begin{align}
    w_t^n 
    = 
    w^n_{t-1}\frac{\pi'(a^n_t|s^n_t)}{\pi_\theta(a^n_t|s^n_t)}, 
    \quad
    \frac{1}{\sum_{n=1}^N w_t^n}\sum_{n=1}^N w_t^n \left(\mathcal{R}(\tau^n_t) + \gamma^t V^{\pi_\theta}(s^n_{t+1})\right)
    \xrightarrow[N \to \infty]{\mathbb{P}}
    V^{\pi'}(s_1).
    \label{eq:RL_SMC_IS_Ws}
\end{align}
We refer to \cite{tsmcts} for derivation and Appendix \ref{app:smc_more_detail} and \cite{chopin2020introduction} for more detail.

%% file: sections/particle_mcts.tex
\section{Particle Monte Carlo Tree Search}
\label{sec:pmcts}
We design Particle Monte Carlo Tree Search (PMCTS) to achieve three desiderata:
(I) Theoretically motivated selection and evaluation to approximate principled policy improvement.
(II) Fully independent parallel selection of multiple trajectories,
while encouraging trajectory diversity.
(III) Correctly account for the presence of remaining duplicate trajectories, to provide principled updates to the tree.
We begin by designing a simple, principled, particle-based batch-parallel variant of MCTS, which retains modern MCTS's interpretation of approximate policy improvement \citep{grill2020monte} under the classical assumptions of exact deterministic dynamics and conditionally independent, unbiased, finite-variance leaf evaluations \citep{UCT}.
Next, we identify the challenges that arise from using DNN evaluations which do not satisfy these classical assumptions.
We then proceed to present PMCTS, which addresses these challenges and achieves the set desiderata.
\subsection{A simple, parallel, particle based Monte Carlo Tree Search algorithm}
We begin with an observation: the principled stochastic selection
policies proposed by \citet{grill2020monte}, such as \(\pi_i\)
(Equation~2), permit multiple trajectories to be sampled independently
in parallel which can produce distinct trajectories without communication between particles.
Motivated by this observation, we design a simple particle-based MCTS algorithm which preserves the policy improvement interpretation of MCTS and reduces evaluation variance.
The algorithm operates in the same way as baseline MCTS, except that each step is now batch-parallelized over particles.

\textbf{Particle-based selection:}
At each iteration $i$, $N$ trajectories 
are sampled independently in parallel from the stochastic selection policy $\pi_i$ (Equation \ref{eq:pi_i}).

\textbf{Particle-based expansion:}
$N$ nodes $s^n_{T+1}$ are expanded in batch-parallel:
\begin{align}
    \big\{
    s^n_{T+1} \sim P(\cdot| s^n_{T}, a^n_T), \quad
    r(s^n_{T}, a^n_T) = \E[R(s^n_{T}, a^n_T)], \quad
    v_i(s^n_{T+1}), \quad
    \pi_\theta(s^n_{T+1})
    \big\}_{n=1}^N.
\end{align}

\textbf{Particle-based backpropagation:}
We now need to compute the particle-based return $\nu_i$ and update to the tree $v_{i+1}$.
Under the classical assumptions, the returns through $s_t$ are identically distributed and independent, with mean $V^{\pi_i}(s_t)$ and finite variance $\sigma_i^2(s_t)$. 
Thus the mean $\nu_i(s_t)$ is conditionally unbiased and has variance $\sigma_i^2(s_t)/N(s_t)$. 
We weight the update to the tree by $N(s_t)$ to maintain $v_{i+1}$
's interpretation as an estimator of the mean across evaluations.
\begin{align}
    \nu_i(s_t) = 
    \frac{1}{N(s_t)}
    \sum_{n=1}^N 
    \mathbbm 1_{s_t \in \tau^n_T} 
    \nu^n(s_t),
    \quad
    v_{i+1}(s_t) = v_{i}(s_t) + \frac{(\nu_i(s_t) - v_{i}(s_t)) N(s_t)}{M_{i+1}(s_t)},
    \label{eq:pmcts_backup}
\end{align}
and $M_{i+1}(s_t) = M_{i}(s_t) + N(s_t)$.
We summarize Simple PMCTS in Algorithm \ref{alg:simple_pmcts} in Appendix \ref{app:pseudocode}.
Action selection and policy improvement follow the description in Section \ref{sec:background}, see Appendix \ref{app:action_selection_and_improvement_survey} for more detail.
Simple PMCTS is principled and batch-parallelizable, addressing the first and second limitations of prior methods, but does not yet account for the combined effect of duplicate trajectories and DNN evaluations.

\subsection{Enter Deep Neural Networks}
\label{sec:dnn_problems}
With DNN evaluations, 
the returns $\nu_i^n(s_t)$ cannot be assumed to be entirely uncorrelated or unbiased.
This is a general problem in search algorithms such as MCTS, SMC for search \citep{piche2019,tsmcts} and MPPI \citep{tdmpc2}.
Despite this, these algorithms are extremely successful in practice, implicitly assuming that \textit{unique} evaluations are \textit{sufficiently} uncorrelated and unbiased.
In discrete-action, parallel search algorithms this problem is significantly exacerbated.
\textit{Multiple} parallel selection trajectories can arrive at \textit{the same} leaf $s_{T+1}$, even when $\pi_i $ is stochastic.
In this case, all trajectories where $s^n_{T+1} = s_{T+1}$ will have \textit{fully} correlated returns.
This causes two main problems:
(I) Duplicate leaf evaluations do not reduce the evaluation's variance and waste compute.
(II) The variance $\mathbb V[\nu_i(s_t)]$ cannot be assumed to be proportional to $1/N(s_t)$ resulting in \textit{incorrectly weighted backpropagation}. 
We proceed to extend Simple PMCTS with mechanisms that address problems (I-II) and design the complete PMCTS algorithm.

\subsection{Weighted Particle Monte Carlo Tree Search (PMCTS)}
\label{sec:weighted_particle_mcts}
Addressing problems (I-II) requires the algorithm to reduce the number of duplicate trajectories while accounting for the remaining duplicate trajectories and estimating $V^{\pi_i}$ before retrospective reweighting and the induced post-expansion trajectory target afterward.
To achieve that, we use SMC within PMCTS, with four mechanisms: \textit{higher-entropy sampling} to increase the number of unique trajectories;
\textit{deduplication} to mitigate the effect of remaining duplicate trajectories and maintain a principled estimate of the selection policy's value;
and the \textit{Effective Sample Size (ESS)} \citep[][]{chopin2020introduction} for principled backpropagation.
In addition, SMC allows the algorithm to address a more general problem in MCTS, over-commitment to catastrophic returns, through a fourth mechanism we call \textit{retrospective reweighting}.
The resulting algorithm is PMCTS, which is described below. PMCTS maintains the same structure of particle-based selection, expansion, backpropagation, except that each particle $n$ is now associated with a weight $w_t^n$ which is updated during selection.

\textbf{Weighted particle-based selection with \textit{higher-entropy sampling}:} To increase trajectory diversity, actions are sampled from a higher-entropy proposal policy $\hat \pi_i$ rather than $\pi_i.$ 
The particle's weight is computed using SIS to maintain a principled estimate of the expectation with respect to $\pi_i$:
\begin{align}
    \label{eq:weight_update}
    w^n_t 
    = 
    w^n_{t-1} \frac{\pi_{i}(a^n_t \mid s^n_t)}{\hat \pi_i(a^n_t \mid s^n_t)}, 
    \quad 
    \quad 
    \hat \pi_i(a_t|s_t) \propto 
    \pi_i(a_t|s_t)^{\frac{1}{\eta(s_t)}}.
\end{align}
This mechanism allows the algorithm to trade off between two sources of variance: the effect of $\hat \pi_i $ diverging from $\pi_i$ and the effect of reduced duplicates to minimize the total variance of the update.

For simplicity, in this work we use a constant rather than a dynamic $\eta(s) := \eta $.
In general, any policy~$\hat \pi_i$ whose support contains the target $\pi_i$'s can be used. 
If particles are allowed to communicate at the cost of parallelization, the temperature can be chosen based on the number of particles per node to reduce variance.
SMC resampling techniques can additionally be incorporated for further variance reduction~\citep{chopin2020introduction}.
However, resampling requires communication and can reduce trajectory
diversity; we leave its investigation in PMCTS to future work.

\textbf{Weighted particle-based backpropagation with \textit{deduplication} and \textit{ESS}:}
After expansion, PMCTS aggregates all duplicate trajectories 
into a single particle while preserving the correct total weight: $w^m_T(s_{T+1}) = \sum_{n=1}^N w^n_T  \mathbbm{1}_{s_{T+1} = s^n_{T+1}} $, which we call \textit{deduplication}.
This results in a new set of particles where all trajectories are unique while correctly preserving the estimate of the expectation with respect to $\pi_i$.
We will refer to this new particle set using the same notation $\{\tau^n_T, w^n_T\}_{n=1}^N$. 
Deduplication is naturally only used when duplicate trajectories have duplicate evaluations, such as with deterministic DNN evaluations.
The backup $\nu_i(s_t)$ is computed using the particles' normalized weights $\bar w^n_T(s_t)$ 
and returns $\nu^n(s_t)$:
\begin{align}
    \nu_i(s_t) = 
   \sum_{n=1}^N \underbrace{
    \bar w^n_T(s_t) \nu^n(s_t) \mathbbm 1_{s_t \in \tau^n_T}
   }_{\text{Weighted particle returns}},
    \quad\quad
    \bar w_T^n (s_t)
    = \frac{
        w_T^n \mathbbm 1_{s_t \in \tau^n_T}
        }{
        \underbrace{{\textstyle\sum_{m=1}^N} w^m_T\mathbbm 1_{s_t \in \tau^m_T}}_{\text{Normalized per node}}
        }.
    \label{eq:pmcts_normalized_estimator}
\end{align}
To update $v_{i+1}$ such that it remains an inverse-variance weighted estimator of the mean, we need to estimate the variance $\mathbb{V}[\nu_i(s_t)].$
A standard choice is the inverse ESS \citep[][]{chopin2020introduction} $ \V\,[\nu_i(s_t)] \approx \frac{1}{\text{ESS}(s_t)} = \sum_{n=1}^N (\bar w^n_T(s_t))^2. $ 
Thus:
\begin{align}
    v_{i+1}(s_t) = v_{i}(s_t) + \frac{(\nu_i(s_t) - v_{i}(s_t)) \text{ESS}(s_t)}{M_{i+1}(s_t)},
    \quad\quad
    M_{i+1}(s_t) = M_{i}(s_t) + \text{ESS}(s_t),
    \label{eq:wmcts_backup}
\end{align}
and $ M(s_{T+1}) = 1 $.
See Appendix \ref{app:derivation_wmcts_backup} for more detail and derivation.

\textbf{Retrospective reweighting:} 
SMC additionally enables PMCTS to address a third problem, general to MCTS: over-commitment to catastrophic returns \citep[see][]{wendelin_over_commitment_to_backprop}.
At parent $s_{T}$, MCTS samples $a_{T} \sim \pi_i(s_{T})$ leading to evaluation $q_i(s_{T},a_{T}) = r_{T} + \gamma v_\phi(s_{T+1}) $.
The action $a_T$ may be catastrophic in terms of its true value $Q^{\pi_\theta}(s_{T},a_T)$ (for example, falling off a cliff), which cannot be known in advance: 
the purpose of the expansion is to gather new information.
In standard MCTS variants, this catastrophic evaluation will then be backpropagated to the root, unnecessarily contaminating all nodes along the trajectory even if the catastrophic action is entirely avoidable.
To address this problem, 
after expansion and before backpropagation, we compute the ``retrospective'' post-expansion policy at each expanded parent $s_T$ and replace the final action's importance ratio using $\pi_{i+1}(a_T\mid s_T)$:
\begin{align}\label{eq:pre_backup_weight_correction}
    w_{T}^n = w^n_{T-1} \frac{\pi_{i+1}(a^n_{T} \mid s^n_{T})}{\hat \pi_i(a^n_{T} \mid s^n_{T})}, 
    \quad
    \pi_{i+1}(a_T|s_T) \propto \pi_\theta(a_T|s_T) \exp \big(\underbrace{\beta_{i+1}(s_T) q_{i+1}(s_T,a_T)}_{\text{Post-expansion.}}\big).
\end{align}
The corrected weights account for the new action evaluations, induce the post-expansion target and downweigh catastrophic trajectories when their associated action becomes less probable relative to the alternatives represented in the batch.

This completes the PMCTS algorithm, which remains principled and batch parallelizable while accounting for the effects of duplicate trajectories. We summarize PMCTS in natural language in Appendix \ref{app:pseudocode:pmcts_summary} for added clarity, along with pseudocode in Algorithm \ref{alg:pmcts}. 

%% file: sections/theoretical_analysis.tex
\section{Theoretical Analysis}
\label{sec:theoretical_analysis}

The main use of MCTS in RL is for policy improvement: search aims to select an action from the best policy it can construct at its root, and, in an Alpha/MuZeroZero-style RL loop, train the prior policy $\pi_\theta$ with this policy as a target. 
\citet{grill2020monte} give MCTS a principled interpretation as approximating policy improvement, showing that canonical selection policies yield visit distributions that track policies improved with respect to tree values $q_i$ and the prior $\pi_\theta$, and that selection can instead use the tracked policy $\pi_i=\I(\pi_\theta,q_i)$ directly.

Following this interpretation, we establish three theoretical results for PMCTS algorithms. 
First, that Simple PMCTS and PMCTS retain MCTS's policy-improvement interpretation: each parallel selection and backpropagation evaluates the improved policy $\pi_i$ directly or by importance weighting, while PMCTS's retrospective correction evaluates the policy improved with respect to the post-expansion evaluation. 
Second, that with a common selection policy, MCTS is the single-particle case of both algorithms, and that the benefit of additional particles is in increased evaluation accuracy. 
And third, a characterization of the conditions under which PMCTS provides benefits over Simple PMCTS, which relax the classical assumptions specified in Section \ref{sec:pmcts} (namely, unbiased finite variance evaluations and exact, deterministic model).

\begin{theorem}[Improved-policy evaluation (informal)]
\label{thm:pmcts_target_eval}
At each iteration, under the classical assumptions, Simple PMCTS unbiasedly estimates $V^{\pi_i}$. 
PMCTS consistently estimates $V^{\pi_i}$ before retrospective reweighting, and the value of its post-expansion trajectory policy afterward.
\end{theorem}

\textit{Intuition.} Simple PMCTS samples from $\pi_i$. PMCTS samples from $\hat\pi_i$ and importance-corrects to $\pi_i$; its final terminal-action correction changes the policy whose value is backed up.

\begin{theorem}[MCTS with increased evaluation accuracy (informal)]
\label{thm:simple_pmcts_eval}
With the same stochastic selection policy $\pi_i$, MCTS, Simple PMCTS, and PMCTS have the same single-particle update at $N=1$. Under the classical assumptions, Simple PMCTS's conditional node-evaluation variance is $\sigma_i^2(s_t)/N_i(s_t)$. PMCTS's errors relative to its pre- and post-retrospective targets are $O_{\mathbb P}(N^{-1/2})$.
\end{theorem}

\textit{Intuition.} One particle's normalized weight is one, even after retrospective reweighting. Additional particles evaluate the same search policy more accurately. 
PMCTS uses importance weights to account for sampling from $\hat\pi_i$, retaining increased accuracy with increased particles.

\begin{theorem}[Benefit of PMCTS over Simple PMCTS (informal)]
\label{thm:pmcts_duplicate_eval}
Suppose duplicate trajectories share an evaluation, while the evaluation errors of distinct trajectories are uncorrelated and have equal variance. 
Before retrospective reweighting, PMCTS has lower evaluation-noise variance than Simple PMCTS when its weighted ESS over distinct trajectories is larger. 
After reweighting, PMCTS also has lower evaluation-noise MSE and variance than Simple PMCTS when its corrected ESS advantage exceeds the effect of adapting its weights to those same evaluations.
\end{theorem}

\textit{Intuition.} 
Repeating a deterministic DNN evaluation of the same trajectory adds no independent information. 
PMCTS's proposal can reach more distinct trajectories while (I) still evaluating $V^{\pi_i}$ and (II) accounting for the presence of duplicate trajectories with the deduplication mechanism.

Together, the results retain MCTS's regularized policy-optimization interpretation while making the parallel backups' targets and evaluation uncertainty explicit. For formal Theorems and proofs see Appendix~\ref{app:pmcts_theory}.
We additionally include a runtime complexity analysis in Appendix \ref{app:complexity_analysis}.

%% file: sections/related_work.tex
\section{Related Work}
\label{sec:related_work}

Runtime-inference scaling of MCTS has been a long-standing objective of AI research \citep{parallel_uct}. 
The canonical approaches are (I) \textit{root parallelization} (running different instances of MCTS in parallel) and (II) \textit{leaf parallelization} (scaling the evaluation of leaves with parallel compute).
Both rely on stochasticity in the selection and/or evaluation to scale and are thus unsuited for runtime inference scaling with modern DNN evaluations and deterministic dynamics and selection.
(III) \textit{Tree parallelization} (parallelizing the search itself), which canonically uses the \textit{virtual visits} heuristic to incentivize the selection of multiple unique trajectories and circumvent these challenges.
When a process arrives at action $a$ in node $s$, it increments a \textit{virtual visit} $M_i^{\text{virtual}}(s,a)$ \citep{virtual_loss}.
When another process arrives at $s$, it computes the deterministic selection policy $\pi^{\text{PUCT}}_i(s)$ using $M_i^{\text{virtual}}(s,a)$ and a heuristic which effectively \textit{predicts} the evaluation of the previous processes in the hope that $\pi^{\text{PUCT}}_i$ selects a different action.
The selected trajectories are then evaluated in parallel, usually batched, and backpropagated, whether they are unique or not, 
so repeated deterministic leaf evaluations increase backed-up visit counts without adding independent evaluation information and the resulting visit distribution does not retain the interpretation of approximating policy improvement with respect to the tree's value.
In addition, to work as intended this selection requires the processes to operate \textit{sequentially} at each \textit{shared} node, both preventing effective batched-parallelization as well as in the worst case resulting in entirely sequential--rather than parallel--selection 
(See Appendices \ref{app:more_related_work} and \ref{app:virtuals} for more detail).
In contrast, PMCTS selects trajectories independently in parallel and retains the explicit policy-improvement interpretation, established in Theorems~\ref{thm:pmcts_target_eval}--\ref{thm:pmcts_duplicate_eval}.

Adjacent to MCTS and PMCTS, SMC has been used directly for search for policy improvement \citep{piche2019,trtpi,tsmcts}.
The resulting SMC-based search algorithms are fundamentally different to MCTS:
They maintain statistics only over $N$ particles rather than memorize the entire search tree. 
As a result, they can only converge to the value of an \textit{improved policy} in the limit of infinite search horizon \citep{tsmcts}, rather than the optimal policy in the case of MCTS \citep{UCT};
their estimates can deteriorate with search depth and limited particle budgets \citep{tsmcts} as we also demonstrate in Appendix \ref{app:additional_results};
and they are not designed to account for duplicate trajectories.
In contrast, PMCTS memorizes the entire tree and scales well across search budgets and particle budgets (see Section \ref{sec:experiments}).
PMCTS is additionally designed to account for the presence of duplicate trajectories with three mechanisms, \textit{increased-entropy selection}, \textit{deduplication}, and \textit{the ESS}, which are novel to this line of work.

%% file: sections/experiments.tex
\section{Experiments}
\label{sec:experiments}
Our objective in this work is to present a novel and effective tree-parallel MCTS algorithm which is both principled as well as suited for batch-parallelization and runtime scaling on the GPU.
In this section, we evaluate PMCTS empirically across three popular MCTS and RL benchmark domains and one LLM inference scaling experiment. 
These encompass: (I) Zero-sum board games (\textit{chess, 19x19 Go, 9x9 Go} and \textit{Gardner chess)} \citep{koyamada2023pgx}; 
(II) Popular~\citep{trtpi,tsmcts} discrete-action single-agent environments (\textit{Snake} and \textit{RubiksCube} \citep{bonnet2024jumanji}); 
(III) Classical continuous control environments (\textit{Ant, HalfCheetah} and \textit{Humanoid} \citep{brax}), where MCTS is extended to continuous actions in the manner of \cite{sampled_mz} (see Appendix \ref{app:continuous_action_mcts} for more detail);
and (IV) scaling the inference of an LLM solving different \textit{Sokoban} puzzles.
As performance metrics, we use 
the popular BayesElo~\citep{bayeselo} metric (measures strength of play in zero-sum games),
the observed averaged return, and the number of Sokoban instances solved, where relevant.
To aggregate across different environments, we take the average BayesElo,  
average normalized return, and
average number of instances solved, respectively.

To evaluate PMCTS as a tree-parallel MCTS algorithm, we begin by comparing its inference-time scaling with number of particles $N$ to the popular tree-parallel MCTS variants Virtual Means and Virtual Losses.
In each environment, all agents use the same $ \pi_\theta $ and $v_{\phi}$ (when applicable) DNNs in the search and the default, popular MCTS hyperparameters \citep{gumbel_mz,mctx}.
The only new parameter introduced with this work is the temperature $\eta$, whose effect is investigated in Figure~\ref{fig:ablations}, based on which we choose $\eta = 1.5$ which is used in all experiments.
For additional experimental and implementation details see Appendices \ref{app:imp_details} and \ref{app:experimental_details}.
\begin{figure}[H]
    \vspace{-0.3cm}
    \centering
    \includegraphics[width=1\linewidth]{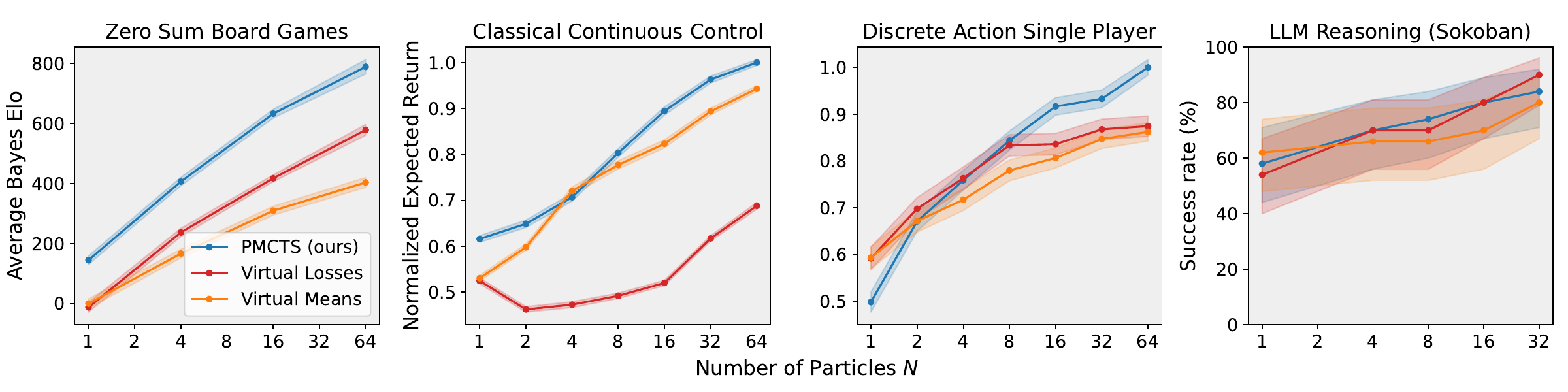}
    \vspace{-0.65cm}
    \caption{
    Scaling of tree-parallel MCTS variants with number of particles $N$ across different domains. 
    BayesElo with 95\% confidence intervals for board games, Gaussian 95\% confidence intervals for discrete and continuous control, and Wilson 95\% confidence intervals for Sokoban.
    }
    \label{fig:scaling}
    \vspace{-0.35cm}
\end{figure}
PMCTS scales well with \(N\) and achieves consistently strong performance relative to both baselines across all four domains.
We additionally investigate the scaling of the different tree-parallel algorithms in an example environment across different budgets $M$
in Figure \ref{fig:scaling_and_sims_gardner} in Appendix \ref{app:additional_results}.
PMCTS scales well across increasing $M$  and outperforms both baselines at every budget pair $M$ and $N$.

To validate that PMCTS indeed provides effective runtime scaling with batch parallelization, we investigate the wall-clock runtime scaling of the different tree-parallel MCTS variants with vectorized, synchronously parallel implementations suited for batch-parallelization in Figure~\ref{fig:runtime_scaling} left and center.
PMCTS's runtime scales very well with additional particles, as expected from an algorithm suited for batch parallelization.
In contrast, the synchronous implementations of Virtual Means and Virtual Losses do not scale well, as expected, due to their inability to select multiple unique trajectories fully independently in parallel and subsequent collapse to sequential rather than parallel selection (see Appendix \ref{app:sync_vs_async} for additional detail).

\begin{figure}[H]
    \vspace{-0.4cm}
    \centering
    \includegraphics[width=1\linewidth]{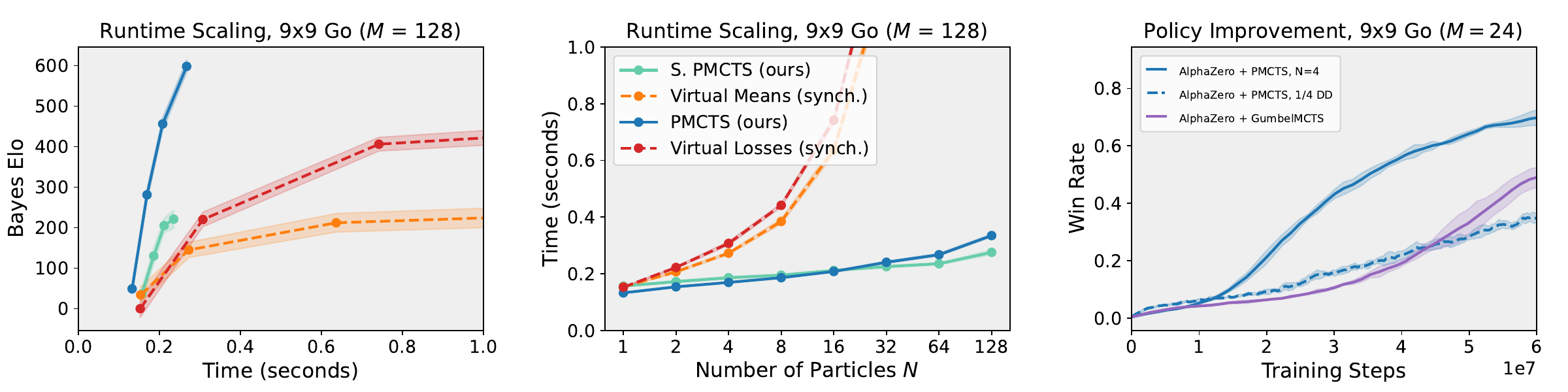}
    \vspace{-0.65cm}
    \caption{
    \textbf{Left: }Runtime scaling, BayesElo with 95\% CI, $N = (1,4,16,64)$ plotted.
    \textbf{Center: }Runtime scaling, 95\% confidence interval across repeated evaluations. 
    \textbf{Right:} Win rate vs. frames during training of AlphaZero with PMCTS and Gumbel MCTS, mean and $95\%$ CI across 3 seeds.
    }
    \label{fig:runtime_scaling}
    \vspace{-0.5cm}
\end{figure}

We investigate PMCTS in an RL training loop and present the results in Figure \ref{fig:runtime_scaling} right. PMCTS with $N=4$ outperforms GumbelMCTS, a modern serial baseline, in sample efficiency. This is expected as PMCTS is evaluating 4 times as many trajectories resulting in better policy improvement. Parallel compute can trade off between parallelizing across states or within a state. We demonstrate this tradeoff with the dashed-line agent, which uses $N=4$ but generates $1/4$ as many targets per training step. All other training parameters, which were tuned for the baseline agent, are kept the same. The dashed-line agent performs similarly to the baseline.

To validate that each mechanism developed in Section \ref{sec:weighted_particle_mcts} indeed contributes to the overall performance of the algorithm, we ablate the different mechanisms in sequence on 9x9 Go ($M=128$), in Figure \ref{fig:ablations} left.
A very large number of mechanism combinations is possible. 
To maintain an informative and reasonable number of experiments, we compare from Simple PMCTS all the way to full PMCTS, with the different mechanisms activated in order, resulting in 6 different agents, each having one more mechanism activated than the previous.
These are:
(I) \textbf{S}imple \textbf{PMCTS} (Algorithm \ref{alg:simple_pmcts}, shortened to \textbf{S}).
(II)~+~\textbf{D}eduplication which removes duplicate trajectories. 
Deduplication computes the weight of the update (Equation \ref{eq:pmcts_backup}, right) using the number of \textit{unique} trajectories instead of $N(s_t)$. 
(III)~+~\textbf{E}SS which weights the update with 1 / ESS as a more principled estimate of the evaluation's variance (Equation \ref{eq:wmcts_backup}).
(IV)~+~\textbf{T}emperature, which increases the number of unique trajectories by increasing the entropy of the selection policy using the temperature $\eta$ (Equation \ref{eq:weight_update}, right).
(V)~+~weight \textbf{C}orrection which uses SIS (Equation~\ref{eq:weight_update}, left) to estimate the expected return with respect to $\pi_i$ rather than $\hat \pi_i$.
(VI) The \textbf{PMCTS} algorithm, which uses retrospective reweighting in addition to all prior mechanisms.
Performance increases consistently with every additional feature.
\begin{figure}[H]
    \vspace{-0.4cm}
    \centering
    \includegraphics[width=1\linewidth]{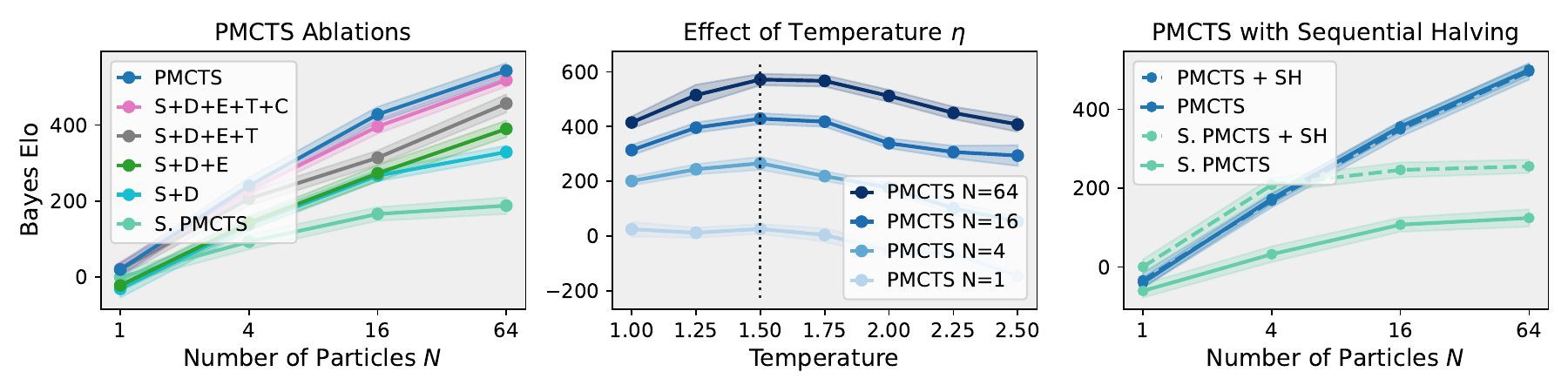}
    \vspace{-0.65cm}
    \caption{
    Ablations and hyperparameter evaluation on 9x9 Go ($M=128$). BayesElo with 95\% CI.
    }
    \label{fig:ablations}
    \vspace{-0.4cm}
\end{figure}
To investigate the effect of the introduced temperature hyperparameter $\eta$, we evaluate the effect of different $\eta$ on the performance of PMCTS across different $N$ in Figure \ref{fig:ablations} center.
The best-performing temperature is $1.5$, the temperature used in all other experiments.
Sequential Halving (SH, \cite{sh}) has been used by \cite{gumbel_mz} and \cite{tsmcts} to improve search at the root.
We investigate its effect on PMCTS, in Figure \ref{fig:ablations} right.
SH, which enforces particle diversity at the root, has a strong positive effect on \textit{Simple} PMCTS. 
However, when combined with PMCTS, it has no additional effect on the scaling.
To reduce unnecessary complexity, we use PMCTS without SH in all experiments.

We include several additional results in Appendix \ref{app:additional_results}, including the individual results aggregated in Figure \ref{fig:scaling}; 
a comparison to \textit{root-parallelization}, to demonstrate that root-parallelization is indeed not effective for runtime scaling with DNNs, as discussed in Sections \ref{sec:intro} and \ref{sec:related_work};
a comparison to the non-MCTS, SMC-based search variants, demonstrating their limitations compared to parallel-MCTS variants;
and an empirical investigation of the trade-off induced by \(\eta\) between trajectory diversity, divergence of the proposal \(\hat{\pi}_i\) from the target \(\pi_i\) and the resulting inverse-ESS variance estimator.

%% file: sections/conclusions.tex
\section{Conclusions and Future Work}
\label{sec:conclusions}
In this work, we introduced \emph{Particle MCTS} (PMCTS), a principled, tree-parallel MCTS algorithm suited for runtime scaling through parallel compute, batch-parallelization and GPU acceleration.
PMCTS uses stochastic selection policies and Sequential Monte Carlo (SMC) to 
select and evaluate multiple trajectories independently in parallel while explicitly encouraging and accounting for trajectory diversity and maintaining principled evaluation with respect to the target policy.
Theoretically, we show that PMCTS retains MCTS's approximate policy improvement interpretation.
Empirically, PMCTS scales well with parallel compute and outperforms or performs comparably to the virtual-visit baselines at matched $(M,N)$ budgets across board games, single-player discrete-action tasks, and continuous control. 
In the 9x9 Go wall-clock comparison using matched JAX/GPU implementations, PMCTS demonstrates the benefits of independently parallel selection.
PMCTS opens up multiple exciting avenues for future work such as more advanced temperature schedules for better proposal policies, incorporating proposals that sample without replacement to increase the number of unique trajectories and 
resampling methods \citep{chopin2020introduction} for different variance tradeoffs.

%% file: appendix/symbols.tex
\section{Acronym and Symbols List}
\label{app:list_of_acronyms}
\begin{tabular}{r|l}
Variable & Meaning \\
\hline
$\gamma$                & discount \\
$\beta_i$               & Q-value scale at iteration $i$ (inverse-temperature) \\
$s_t$                   & state at step $t$ \\
$\tau^n_t$              & trajectory at step $t$ for particle $n$ \\
$T$                     & last step \\
$\pi_\theta(s_t)$       & prior policy network at state $s_t$ \\
$v_\phi(s_t)$           & value network prediction at state $s_t$ \\
$\pi'(a_t|s_t)$         & target improved-policy for action $a_t$ at state $s_t$ \\
$\pi_i(a_t|s_t)$        & selection policy at iteration $i$ for action $a_t$ at state $s_t$ \\ 
$\hat{\pi}_i(a_t|s_t)$  & sampling (e.g. \textit{proposal}) policy at state $s_t$ and iteration $i$ \\
$M_i(s_t)$              & visit count (or effective visit count in PMCTS) in the tree for state $s_t$ at iteration $i$
                        \\ 
$v_i(s_t)$              & value estimate in the tree for state $s_t$ at iteration $i$ \\
$v(s_{T+1})$              & value prediction at leaf $s_{T+1}$ in the tree during expansion \\
$q_i(s_t,a_t)$          & Q-value estimate in the tree for state $s_t$ and action $a_t$ at iteration $i$ \\
$N(s_t)$                & number of particles passing through state $s_t$ \\
$ESS(s_t)$              & effective sample size for particles that passed through state $s_t$ \\
$w^n_t$                 & weight of particle $n$ at step $t$ \\
$\bar w_T^n (s_t)$           & weight of particle $n$ at step $T$ normalized for state $s_t$ \\
$\nu^n(s_t)$            & value update at state $s_t$ by particle $n$ \\
$\nu(s_t)$              & value update at state $s_t$ across all particles \\
$\eta(s_t)$          & possibly state-dependent temperature \\
\end{tabular}

%% file: appendix/pseudocode.tex
\section{Pseudocode}
\label{app:pseudocode}
\subsection{Simple PMCTS}
\begin{algorithm}[H]
    \caption{Simple PMCTS}
    \label{alg:simple_pmcts}
    \begin{algorithmic}[1]
        \Require Number of particles $N$, Number of iterations $M$, Root state $s_1$, Environment model $\mathcal{M}$, DNNs $\pi_\theta, v_\phi.$
        \State Initialize a tree with capacity for $N M + 1$ nodes.
        \State Initialize root using predictions $\pi_\theta(s_1)$ and $v_\phi(s_1)$.
        
        \For{$i \gets 1$ to $M$}
            \State Sample $\{\tau^n_T = (s_1, a^n_1, ..., s^n_T, a^n_T)\}_{n=1}^N$ from $ \pi_i, \mathcal{M}$ independently in parallel. \Comment{Selection.}
            \State 
            Expand and add to the tree up to $N$ unique leaves in parallel: \Comment{Expansion.}
            \begin{align*}
                \big\{
                s^n_{T+1} \sim P(s^n_{T}, a^n_T), 
                \,
                r(s^n_{T}, a^n_T) = \E[R(s^n_{T}, a^n_T)], 
                \,
                v(s^n_{T+1}), 
                \,
                \pi_\theta(s^n_{T+1})
                \big\}_{n=1}^N.
            \end{align*}
            \State For each node $s_t$ compute $\nu_i(s_t)$ (Equation \ref{eq:pmcts_backup}, left): 
            \Comment{Backpropagation.}
            \begin{align*}
                \nu_i(s_t) = 
                \frac{1}{N(s_t)}
                \sum_{n=1}^N 
                \mathbbm 1_{s_t \in \tau^n_T} 
                \nu^n(s_t),
            \end{align*}
            \State And update the tree (Equation \ref{eq:pmcts_backup}, right):
            \begin{align*}
                v_{i+1}(s_t) = v_i(s_t) + \frac{(\nu_i(s_t) - v_i(s_t)) N(s_t)}{M_i(s_t) + N(s_t)}, \quad M_{i+1}(s_t) = M_i(s_t) + N(s_t).
            \end{align*}
        \EndFor
        \State \Return:  
        \Comment{See Appendix \ref{app:action_selection_and_improvement_survey}, Equation \ref{eq:actual_search_return_scores}}
            \begin{align*}
                \pi_{search}(a|s_1) &= \I(\pi_\theta, q_{M+1}), \quad V_{search}(s_1) = \sum_{a \in \A} \bar \pi(a|s_1)q_{M+1}(s_1,a), \quad a \sim \bar \pi(s_1). \\
                \bar \pi(a|s_1) &\propto 
                    \begin{cases}
                        \pi_{search}(a|s_1) \quad \text{if:} \quad  M(s_1,a) > 0 \\
                        0 \quad \quad \quad \quad \quad \quad  \text{if:} \quad  M(s_1,a) = 0
                    \end{cases}
            \end{align*}
    \end{algorithmic}
\end{algorithm}

\subsection{PMCTS}
\label{app:pseudocode:pmcts_summary}
For added clarity, we include a brief summary of PMCTS in natural language in addition and following the pseudocode.

\begin{algorithm}[H]
    \caption{PMCTS}\label{alg:pmcts}
    \begin{algorithmic}[1]
        \Require Number of particles $N$, Number of iterations $M$, Root state $s_1$, Environment model $\mathcal{M}$, DNNs $\pi_\theta, v_\phi$.        
        \State Initialize a tree with capacity for $N M + 1$ nodes.
        \State Initialize root using predictions $\pi_\theta(s_1)$ and $v_\phi(s_1)$.
        
        \For{$i \gets 1$ to $M$}
            \For{$t \gets 1$ to $T$ for $n \in N$ particles independently in parallel} \Comment{Weighted Selection}
                \State Sample and append to trajectory $\tau_{t-1}$: \Comment{\textit{SMC's Sampling}}
                \begin{align*}
                    a^n_t \sim \hat \pi_i(a^n_t|s^n_t), 
                    \quad 
                    s^n_{t+1} \sim P(\cdot|s^n_t, a^n_t), 
                \end{align*}
                \State Update particle weights (Equation \ref{eq:weight_update}): \Comment{\textit{SMC's Correction}}
                \begin{align*}
                    w^n_t 
                    = 
                    w^n_{t-1} \frac{\pi_{i}(a^n_t \mid s^n_t)}{\hat \pi_i(a^n_t \mid s^n_t)},
                \end{align*}
            \EndFor
            \State 
            Expand and add to the tree up to $N$ unique leaves in parallel: \Comment{Expansion.}
            \begin{align*}
                \big\{
                s^n_{T+1} \sim P(s^n_{T}, a^n_T), 
                \,
                r(s^n_{T}, a^n_T) = \E[R(s^n_{T}, a^n_T)], 
                \,
                v(s^n_{T+1}), 
                \,
                \pi_\theta(s^n_{T+1})
                \big\}_{n=1}^N.
            \end{align*}
            \State For every expanded pair \((s_T^n,a_T^n)\), with updated normalizer $Z_{i+1},$ compute:
\[
\pi_{i+1}(a_T^n\mid s_T^n)
=
\pi_\theta(a_T^n\mid s_T^n)
\exp\!\left(
\beta_{i+1}(s_T^n)q_{i+1}(s_T^n,a_T^n)
\right)
/
Z_{i+1}
\]
            \State Update weights $w^n_T$ with respect to $\pi_{i+1}(s_T)$ (Equation \ref{eq:pre_backup_weight_correction}): \Comment{Retrospective reweighting}
            \begin{align*}
                w_T^n \gets 
                w_T^n \frac{\pi_{i+1} (a^n_{T} \mid s^n_{T})}{\pi_i (a^n_{T} \mid s^n_{T})}
            \end{align*}
            \State Merge duplicates, preserving weight $w^n_T(s_{T+1}) = \sum_{m=1}^N w^m_T  \mathbbm{1}_{s_{T+1} = s^m_{T+1}} $ per unique trajectory and setting duplicates to zero.
            \State Normalize weights with respect to each other at all nodes $s_t$:
            \begin{align*}
                \bar w_T^n (s_t)
                = \frac{
                    w_T^n \mathbbm 1_{s_t \in \tau^n_T}
                    }{
                    {\textstyle\sum_{m=1}^N} w^m_T\mathbbm 1_{s_t \in \tau^m_T}
                    }.
            \end{align*}
            \State Compute $\nu_i(s_t)$ (Equation \ref{eq:pmcts_normalized_estimator}, left) and the ESS for each node $s_t$:\Comment{Weighted Backprop.}
            \begin{align*}
                \nu_i(s_t) = 
                   \sum_{n=1}^N \bar w^n_T(s_t) \mathbbm 1_{s_t \in \tau^n_T}\nu^n(s_t), 
                   \quad 
                   ESS(s_t) = \frac{1}{\sum_{n=1}^N (\bar w^n_T(s_t))^2}
            \end{align*}
            \State Update every modified node $s_t$ in the tree (Equation \ref{eq:wmcts_backup}):
            \begin{align*}
                v_{i+1}(s_t) = v_i(s_t) + \frac{(\nu_i(s_t) - v_i(s_t)) \text{ESS}(s_t)}{M_i(s_t) + \text{ESS}(s_t)},
                \quad
                M_{i+1}(s_t) = M_i(s_t) + \text{ESS}(s_t).
            \end{align*}
        \EndFor
        \State \Return:  
        \Comment{See Appendix \ref{app:action_selection_and_improvement_survey}, Equation \ref{eq:actual_search_return_scores}}
            \begin{align*}
                \pi_{search}(a|s_1) &= \I(\pi_\theta, q_{M+1}), \quad V_{search}(s_1) = \sum_{a \in \A} \bar \pi(a|s_1)q_{M+1}(s_1,a), \quad a \sim \bar \pi(s_1). \\
                \bar \pi(a|s_1) &\propto 
                    \begin{cases}
                        \pi_{search}(a|s_1) \quad \text{if:} \quad  M(s_1,a) > 0 \\
                        0 \quad \quad \quad \quad \quad \quad  \text{if:} \quad  M(s_1,a) = 0
                    \end{cases}
            \end{align*}
    \end{algorithmic}
\end{algorithm}

\textbf{PMCTS summarized:} At each parallel MCTS iteration $i,$ $N$ particles sample trajectories independently in parallel from a higher-entropy stochastic-selection policy $\hat \pi_i$ (Equation \ref{eq:weight_update}, right).
At each step during selection, the particles' weights are updated with SIS, maintaining a principled estimate of the distribution over trajectories with respect to $\pi_i$ (Equation \ref{eq:weight_update}, left).
When all particles arrive at a leaf, up to $N$ leaves are expanded in on batch parallel pass of batch size $N$, as in Simple PMCTS.
In our practical implementation, we expand all $N$ leaves, but only keep track of each unique trajectory once, so the tree does not have duplicate leaves.
Retrospective reweighting is then applied, updating the particles' weights to estimate the distribution with respect to $\pi_{i+1}$ rather than $\pi_i$ (Equation \ref{eq:pre_backup_weight_correction}).
Deduplication is applied, removing duplicate trajectories while maintaining the correct importance weight per unique trajectory.
Backpropagation is applied, computing the returns $\nu_i$ and the ESS and using them to update $v_{i+1}$ (Equations \ref{eq:pmcts_normalized_estimator} and \ref{eq:wmcts_backup} respectively).
At this point, the iteration is complete, and the algorithm proceeds to the next iteration.

%% file: appendix/derivations.tex
\section{Derivations}
\label{app:derivations}

\subsection{Derivation of the numerically stable weighted average}
\label{app:in_place_w_a_derivation}
\begin{align}
    v_{i+1}(s_t) &= \frac{\nu(s_t) N_i(s_t) + v_i(s_t) M_i(s_t)}{M_i(s_t) + N(s_t)}\\
    &= \frac{\nu(s_t) N_i(s_t) + v_i(s_t) M_i(s_t) + v_i(s_t) N_i(s_t) - v_i(s_t) N_i(s_t)}{M_i(s_t) + N(s_t)}\\
    &= \frac{v_i(s_t) M_i(s_t) + v_i(s_t) N_i(s_t)}{M_i(s_t) + N_i(s_t)} + \frac{\nu(s_t) N_i(s_t) - v_i(s_t) N_i(s_t)}{M_i(s_t) + N_i(s_t)}\\
    &= v_i(s_t) + \frac{(\nu(s_t) - v_i(s_t)) N_i(s_t)}{M_i(s_t) + N_i(s_t)}.
\end{align}

\subsection{Derivation of the particle-based backpropagation step in PMCTS}
\label{app:derivation_wmcts_backup}
\subsubsection{Derivation of the variance estimator}
\label{app:derivation_of_pmcts_variance}
We analyze the variance of the PMCTS estimator and construct an estimator of it by considering the set of $N$ particles contributing to the update of node $s_t$. 
The estimator is defined as the weighted sum:
\begin{align}
    \mathbb{V}[\nu(s_t)] = \mathbb{V} \left[ \sum_{n=1}^N \bar{w}^n_T \nu^n(s_t) \right]
\end{align}
We assume that we have merged all duplicate particles into one particle with the total weight and all trajectories are unique.
After deduplication, each retained particle represents one unique trajectory with its aggregate importance weight. 
Before normalization, the proposal trajectories are sampled independently in our SIS construction. 
Their self-normalized contributions are nevertheless coupled through the common random denominator, and evaluation errors of distinct trajectories need not be independent. 
As a practical approximation, we treat the normalized weights as fixed and neglect
cross-trajectory covariance.
Under this assumption, and by treating the weights as approximately constant relative to the variance of the returns, we obtain:
\begin{align}
    \mathbb{V}[\nu(s_t)] \approx \sum_{n=1}^N \mathbb{V}[\bar{w}^n_T \nu^n(s_t)] \approx \sum_{n=1}^N (\bar{w}^n_T)^2 \mathbb{V}[\nu^n(s_t)]
\end{align}
This formulation mirrors the derivation of the Effective Sample Size (ESS) \citep{chopin2020introduction}, where the variance of a weighted importance sampling estimator is shown to scale with $\sum (\bar{w}^n_t)^2$. The ESS approximation is widely used as a proxy for estimator quality when the true covariance structure is intractable.

Assuming the variance of the return for any single trajectory is bounded by $\mathbb{V}[\nu^n(s_t)] \leq \sigma^2$, we obtain the following bound within the fixed-weight, negligible-covariance approximation:
\begin{align}
    \mathbb{V}[\nu(s_t)] 
    \approx 
    \sum_{n=1}^N (\bar{w}^n_T)^2 \mathbb{V}[\nu^n(s_t)] 
    \leq 
    \sigma^2 \sum_{n=1}^N (\bar{w}^n_T)^2
    =
    \sigma^2\frac{1}{ESS(s_t)} 
\end{align}
This suggests that the variance of the node update is approximately proportional to the sum of the squared weights, $\sum_{n=1}^N (\bar{w}^n_t)^2$, allowing us to use this term as a variance proxy in the tree search.

\subsubsection{Derivation of the node value update}
\label{app:derivation_of_pmcts_value_update}
The weight of the update to the variance-minimizing unbiased estimator of the mean is the inverse variance.
Under the approximation in Equation~(17), and after removing the common scale \(\sigma^2\), the variance estimator is:
\[
\sum_{n=1}^{N}\bigl(\bar w_T^n(s_t)\bigr)^2,
\]
and its inverse is
\[
\frac{1}{
\sum_{n=1}^{N}\bigl(\bar w_T^n(s_t)\bigr)^2
}
=
\mathrm{ESS}(s_t).
\]
We therefore use \(\mathrm{ESS}(s_t)\) to approximate the inverse-variance of the update and weight the particle-based update.

We can now derive the update:
\begin{align}
    v_{i+1}(s_t) &= \frac{\nu(s_t) \text{ESS}(s_t) + v_i(s_t) M_i(s_t)}{M_i(s_t) + \text{ESS}(s_t)}
    \\
    &= \frac{\nu(s_t) \text{ESS}(s_t) + v_i(s_t) M_i(s_t) + v_i(s_t) \text{ESS}(s_t) - v_i(s_t) \text{ESS}(s_t)}{M_i(s_t) + \text{ESS}(s_t)}
    \\
    &= \frac{v_i(s_t) M_i(s_t) + v_i(s_t) \text{ESS}(s_t)}{M_i(s_t) + \text{ESS}(s_t)} + \frac{\nu(s_t) \text{ESS}(s_t) - v_i(s_t) \text{ESS}(s_t)}{M_i(s_t) + \text{ESS}(s_t)}
    \\
    &= v_i(s_t) + \frac{(\nu(s_t) - v_i(s_t)) \text{ESS}(s_t)}{M_i(s_t) + \text{ESS}(s_t)}.
\end{align}

\subsection{Approximating expectations under the target in Sequential Monte Carlo}
\label{app:smc_more_detail}
There are multiple approaches to approximating expectations under the target using the particles $\{\tau_t^n, w_t^n\}_{n=1}^N$ of SMC, for some example function $\mathcal{R}$ (in RL terms, we can think of $\mathcal{R}(\tau_T^n)$ as the return along the trajectory).

The standard approach is to use the \textit{self-normalizing} estimator:
\begin{align}
    \frac{1}{\sum_{n=1}^N w_t^n}\sum_{n=1}^N w_t^n \mathcal{R}(\tau^n_t)
    \approx 
    \mathbb{E}_{\pi'}[R(\tau^n_t)] 
    =
    V^{\pi'}(s_1).
    \label{eq:biased_estimator}
\end{align}
It normalizes with respect to the sum of the \textit{weights} of the particles $\sum_{n=1}^N w_t^n$.
It is the more common of the options, as it allows for SIS even with access only to the probability mass / density functions $\pi'(a|s), \pi_\theta(a|s)$ up to a proportionality constant.
This estimator however is biased (although consistent), introducing bias of order $\O(1/N)$ \citep{cardoso2022br}, because the ratio of two unbiased estimators: $ \sum_{n=1}^N w_t^n \mathcal{R}(\tau^n_t) / \sum_{n=1}^N w_t^n,  $ is not itself unbiased \citep{chopin2020introduction}.

However, when we have access to the full densities / masses $\pi'(a|s), \pi_\theta(a|s)$, we can also use the estimator (Equation \ref{eq:RL_SMC_IS_Ws}):
\begin{align}
    \frac{1}{N}\sum_{n=1}^N w_t^n \mathcal{R}(\tau^n_t)
    \approx 
    \mathbb{E}_{\pi'}[R(\tau^n_t)] 
    =
    V^{\pi'}(s_1).
    \label{eq:un_biased_estimator}
\end{align}
This is an \textit{unbiased} estimator of expectations under the function $\mathcal{R}$.

In practice, the first estimator (Equation \ref{eq:biased_estimator}) is often preferred in practice because it is numerically stable and
can have lower variance, at the cost of an \(O(1/N)\) bias under
standard regularity conditions, and we use it in the practical algorithm.

\subsection{Unbiased estimates of the value at any depth}
\label{app:unbiased_at_every_depth}
The unnormalized estimator provides unbiased estimates of the expectation at the root $s_1$.
However, we cannot simply use the same weights for any node under the root (the way we could for the self-normalizing estimator) and expect unbiased expectations, because the weights will be proportional to the probability of reaching those nodes. 
As a result, unlikely trajectories will have value estimates close to zero (because their weights will be close to zero) and using them directly to update the tree will effectively set all the values closer to zero.

This can be addressed however, by maintaining an array of weights per particle $n$, which save a weight per depth $t$.
At each depth $t$ the new-weight to depth $t$ is initialized with $1.0,$ and is updated in the same manner as all other weights at all depths.
As a result, a different weight is available per particle per depth to compute the expectation at each depth (i.e. each node along the trajectory).
This can effectively be thought of as "restarting" the SMC process at each depth guaranteeing that its properties apply for each depth and not just the root.

In practice, PMCTS uses the self-normalized estimator. 
Unlike the unnormalized estimator, it does not require a separate weight array at every depth because the common prefix weight cancels in the node-wise normalization. 
Bias-reduced variants of the self-normalized estimator also exist \citep{cardoso2022br}.


%% file: appendix/proofs.tex
\section{Formal Theoretical Analysis}
\label{app:pmcts_theory}

\subsection{Improved-policy evaluation}
\label{app:pmcts_target_eval}

We first make explicit the improvement objective used by the regularized instance of Equation~\ref{eq:pi_i}. Assume a finite action space $\A$, $\pi_\theta(a\mid s_t)>0$, and $\beta_i(s_t)>0$. Then $\pi_i=\I_{\mathrm{reg}}(\pi_\theta,q_i)$ uniquely solves
\begin{align}
\label{eq:pmcts_local_objective}
\pi_i(\cdot\mid s_t)
&=\underset{\mu\in\Delta(\A)}{\arg\max}\;
\left\{\sum_{a\in\A}\mu(a)q_i(s_t,a)
-\frac{1}{\beta_i(s_t)}
D_{\mathrm{KL}}\!\left(\mu\,\|\,\pi_\theta(\cdot\mid s_t)\right)\right\}.
\end{align}
Indeed, setting $Z_i(s_t)=\sum_a\pi_\theta(a\mid s_t)e^{\beta_i(s_t)q_i(s_t,a)}$, the objective at any $\mu$ equals
\begin{align}
\label{eq:pmcts_local_objective_identity}
\sum_a\mu(a)q_i(s_t,a)
-\frac{D_{\mathrm{KL}}(\mu\|\pi_\theta)}{\beta_i(s_t)}
&=\frac{\log Z_i(s_t)-D_{\mathrm{KL}}(\mu\|\pi_i)}{\beta_i(s_t)}.
\end{align}
The right side is uniquely maximized at $\mu=\pi_i$. Comparing its objective with that of $\pi_\theta$ additionally gives
\begin{align}
\label{eq:pmcts_local_objective_gain}
\sum_a\bigl(\pi_i(a\mid s_t)-\pi_\theta(a\mid s_t)\bigr)q_i(s_t,a)
&\geq\frac{D_{\mathrm{KL}}(\pi_i\|\pi_\theta)}{\beta_i(s_t)}\geq0.
\end{align}
The divergences in these expressions are evaluated at $s_t$. This is improvement against the search values $q_i$; the backup results below do not require a particular choice of the improvement operator $\I$.

Fix a pre-batch tree. Let $\mathcal T_i$ be its finite set of complete root-to-frontier or terminal trajectories. The target $\pi_i$ selects actions in the tree and follows $\pi_\theta$ after the frontier. Write $p_i(\tau)$ and $\hat p_i(\tau)$ for the probabilities of $\tau\in\mathcal T_i$ under $\pi_i$ and PMCTS's proposal $\hat\pi_i$, respectively. 
Equation~\ref{eq:weight_update} supplies the pre-retrospective importance ratio $\omega_i(\tau)=p_i(\tau)/\hat p_i(\tau)$. 
To describe Equation~\ref{eq:pre_backup_weight_correction}, let $\pi_{i+1}^{(N)}$ be the policy at the newly evaluated parent nodes after expanding a batch of $N$ particles. 
If the last action $(s_T,a_T)$ of $\tau$ was expanded, define its terminal-action correction by
\begin{align}
\label{eq:pmcts_ret_factor}
d_{i,N}(\tau)
&=\frac{\pi_{i+1}^{(N)}(a_T\mid s_T)}{\pi_i(a_T\mid s_T)}.
\end{align}
Set $d_{i,N}(\tau)=1$ for trajectories ending at a terminal state. The corrected weights target the trajectory distribution
\begin{align}
\label{eq:pmcts_ret_distribution}
p_{i,N}^{\mathrm{ret}}(\tau)
&=\frac{p_i(\tau)d_{i,N}(\tau)}
{\sum_{\tau'\in\mathcal T_i}p_i(\tau')d_{i,N}(\tau')}.
\end{align}
Its conditional action probabilities at history-specific tree nodes define $\pi_{i,N}^{\mathrm{ret}}$, followed by $\pi_\theta$ beyond the frontier. Equation~\ref{eq:pre_backup_weight_correction} corrects only a trajectory's last expanded action; therefore $\pi_{i,N}^{\mathrm{ret}}$ need not be $\pi_{i+1}^{(N)}$ at all expanded nodes.

\begin{theorem}[Improved-policy evaluation (formal)]
\label{thm:pmcts_target_eval_formal}
Fix an iteration $i$ and condition on the pre-batch search history. Hold the finite pre-batch tree, $\pi_i$, and $\hat\pi_i$ fixed as $N$ varies. Assume a true deterministic model, bounded rewards, finite action spaces, and $\hat\pi_i(a\mid s)>0$ wherever $\pi_i(a\mid s)>0$, with positive target probabilities on searched actions. Each particle independently selects a trajectory under the respective selection policy: $\pi_i$ for Simple PMCTS, $\hat\pi_i$ for PMCTS, without resampling. Conditional on its selected frontier, its evaluation has mean $V^{\pi_\theta}$ and a finite second moment, independently across particles; an exact deterministic evaluation may instead be shared. Every sampled frontier action is expanded. At every pre-existing node $s_t$ with positive reach probability under $\pi_i$, for every $n\geq1$,
\begin{equation}
\label{eq:pmcts_target_results}
\begin{aligned}
\mathbb E[\nu_i^{\mathrm S}(s_t)\mid N_i^{\mathrm S}(s_t)=n]
&=V^{\pi_i}(s_t),\\
\nu_{i,\mathrm{pre}}^{\mathrm P,(N)}(s_t)
&\xrightarrow[N\to\infty]{\mathrm{a.s.}}V^{\pi_i}(s_t),\\
\nu_{i,\mathrm{ret}}^{\mathrm P,(N)}(s_t)
-V^{\pi_{i,N}^{\mathrm{ret}}}(s_t)
&\xrightarrow[N\to\infty]{\mathrm{a.s.}}0.
\end{aligned}
\end{equation}
Here $\nu_i^{\mathrm S}$ is the new-batch average in Equation~(6), and the two PMCTS terms are the new-batch self-normalized backups in Equation~(8) before and after Equation~(10). The corrected target may depend on the same batch, and neither finite-$N$ PMCTS backup is generally unbiased.
\end{theorem}

\begin{proof}
For Simple PMCTS, let $Z_n=\mathbf1_{\{s_t\in\tau_T^n\}}$ and set $Y_n=Z_n\nu_i^n(s_t)$, with $Y_n=0$ if $Z_n=0$. The pairs $(Y_n,Z_n)$ are i.i.d. Conditional on $Z_n=1$, actions are sampled from $\pi_i$ up to the first frontier action, and the leaf continuation has the correct conditional mean under $\pi_\theta$. Hence $\mathbb E[Y_n\mid Z_n=1]=V^{\pi_i}(s_t)$. Given $N_i^{\mathrm S}(s_t)=n\geq1$, Equation~(6) averages $n$ such returns. This proves its conditional expectation in Equation~\ref{eq:pmcts_target_results}.

For PMCTS, let $m_t(\tau)$ be the true expected return from $s_t$ conditional on $\tau\ni s_t$, including the $\pi_\theta$ continuation. For each $\tau\in\mathcal T_i$ define
\begin{equation}
\label{eq:pmcts_path_contributions}
\begin{aligned}
C_N(\tau)
&=\frac1N\sum_{n=1}^{N}\mathbf1_{\{\tau_T^n=\tau\}},\\
H_{t,N}(\tau)
&=\frac1N\sum_{n=1}^{N}
\mathbf1_{\{\tau_T^n=\tau\}}\nu_i^n(s_t)\qquad(s_t\in\tau).
\end{aligned}
\end{equation}
Independence and the finite second moment imply, by the strong law and finiteness of $\mathcal T_i$, that almost surely and simultaneously for every such path and node,
\begin{equation}
\label{eq:pmcts_path_limits}
\begin{aligned}
C_N(\tau)&\longrightarrow\hat p_i(\tau),\\
H_{t,N}(\tau)&\longrightarrow\hat p_i(\tau)m_t(\tau).
\end{aligned}
\end{equation}
The same limits hold for a shared evaluation when that evaluation equals the exact conditional continuation value. Before retrospective correction, the numerator of the self-normalized backup at $s_t$ is $\sum_{\tau\ni s_t}\omega_i(\tau)H_{t,N}(\tau)$, and its denominator is $\sum_{\tau\ni s_t}\omega_i(\tau)C_N(\tau)$. Their limits are $\sum_{\tau\ni s_t}p_i(\tau)m_t(\tau)$ and $\sum_{\tau\ni s_t}p_i(\tau)>0$, respectively. Their ratio equals $V^{\pi_i}(s_t)$.

After correction, the same two path sums acquire factors $d_{i,N}(\tau)$. These factors can depend on all returns in the batch, but they satisfy $0\leq d_{i,N}(\tau)\leq1/\pi_i(a_T\mid s_T)$ for a frontier trajectory, and equal one for a terminal trajectory. Since $\mathcal T_i$ is finite, the differences between the corrected empirical sums and their population path sums vanish almost surely, uniformly over the possible values of $\pi_{i+1}^{(N)}$.

The population denominator for $s_t$ is bounded below by a fixed positive constant for all sufficiently large $N$. To see this, follow a positive-probability path from $s_t$ to a deepest pre-batch node $s^\star$ that is either terminal or has no expanded outgoing action. A terminal trajectory contributes its positive path probability. Otherwise, every action at $s^\star$ has positive proposal probability and therefore occurs in the batch eventually, almost surely. All these actions are then expanded. If $\rho$ is the prefix ending at $s^\star$, their combined corrected target mass is $p_i(\rho)\sum_a\pi_{i+1}^{(N)}(a\mid s^\star)=p_i(\rho)>0$. Thus both corrected ratios are well defined asymptotically. By Equation~\ref{eq:pmcts_ret_distribution}, the ratio of the population path sums is exactly $V^{\pi_{i,N}^{\mathrm{ret}}}(s_t)$. Uniform convergence of the sums proves the last limit in Equation~\ref{eq:pmcts_target_results}, despite the data-dependent correction. Finally, self-normalized importance ratios and the data-dependent correction need not be unbiased at finite $N$.
\end{proof}

\subsection{Single-particle MCTS and evaluation accuracy}
\label{app:simple_pmcts_eval}

\begin{theorem}[MCTS with increased evaluation accuracy (formal)]
\label{thm:simple_pmcts_eval_formal}
Under the assumptions of Theorem~\ref{thm:pmcts_target_eval_formal}, use the same stochastic selection policy $\pi_i$ for MCTS and Simple PMCTS, and set PMCTS's proposal to $\hat\pi_i=\pi_i$ for the $N=1$ comparison. Assume its retrospective factor is positive, as with the full-support regularized policy in Equation~\ref{eq:pi_i}. Then all three algorithms have the same single-trajectory selection, expansion, and node update at $N=1$, including PMCTS's retrospective step. For any pre-existing node $s_t$ with positive reach probability, let $N_i^{\mathrm S}(s_t)$ denote the number of Simple PMCTS particles visiting it in the current batch (the $N(s_t)$ of Equation~(6)) and let $\sigma_i^2(s_t)=\mathbb V[\nu_i^1(s_t)\mid s_t\in\tau_T^1]$. For every $n\geq1$, and for either full-PMCTS proposal satisfying Theorem~\ref{thm:pmcts_target_eval_formal},
\begin{equation}
\label{eq:pmcts_simple_variance}
\begin{aligned}
\mathbb V[\nu_i^{\mathrm S}(s_t)\mid N_i^{\mathrm S}(s_t)=n]
&=\frac{\sigma_i^2(s_t)}{n},\\
\nu_{i,\mathrm{pre}}^{\mathrm P,(N)}(s_t)-V^{\pi_i}(s_t)
&=O_{\mathbb P}(N^{-1/2}),\\
\nu_{i,\mathrm{ret}}^{\mathrm P,(N)}(s_t)
-V^{\pi_{i,N}^{\mathrm{ret}}}(s_t)
&=O_{\mathbb P}(N^{-1/2}).
\end{aligned}
\end{equation}
Moreover, $N_i^{\mathrm S}(s_t)\to\infty$ and $\nu_i^{\mathrm S}(s_t)\to V^{\pi_i}(s_t)$ almost surely as $N\to\infty$, simultaneously at every such pre-existing node. The two PMCTS rates concern the respective targets identified in Theorem~\ref{thm:pmcts_target_eval_formal}.
\end{theorem}

\begin{proof}
At $N=1$ all three algorithms select the same trajectory when their selection policies agree, and they expand and evaluate its frontier action identically. Simple PMCTS's visiting-node count in Equation~\ref{eq:pmcts_backup} is one. 
PMCTS's per-node normalized weight and ESS in Equations~\ref{eq:pmcts_normalized_estimator}-\ref{eq:wmcts_backup} are both one; multiplying its sole particle's weight by the retrospective factor in Equation~\ref{eq:pre_backup_weight_correction} does not change its normalized return or ESS. 
Both variants therefore perform the same one-trajectory node update as MCTS. This equivalence uses no additional action-selection heuristic.

The i.i.d. pairs $(Y_n,Z_n)$ from the proof of Theorem~\ref{thm:pmcts_target_eval_formal} imply that, conditional on $N_i^{\mathrm S}(s_t)=n$, the $n$ visiting returns are independent and identically distributed, with variance $\sigma_i^2(s_t)$. Their average has the variance in Equation~\ref{eq:pmcts_simple_variance}. Write $p_t=\mathbb E[Z_n]>0$. The strong law gives $N_i^{\mathrm S}(s_t)/N\to p_t$ and $\sum_{n=1}^{N}Y_n/N\to p_tV^{\pi_i}(s_t)$ almost surely. Dividing these limits proves the convergence of the batch average. The pre-batch tree is finite, so the conclusions hold simultaneously at its positive-probability nodes.

For PMCTS, each coordinate $C_N(\tau)$ and $H_{t,N}(\tau)$ in Equation~\ref{eq:pmcts_path_contributions} is an average of independent variables with finite variance. Hence its difference from the corresponding limit in Equation~\ref{eq:pmcts_path_limits} is $O_{\mathbb P}(N^{-1/2})$. There are finitely many paths, so the largest such difference has the same order. Before retrospective correction, the importance ratios $\omega_i(\tau)$ are fixed and bounded, and the limiting node denominator is positive. Taking the ratio gives the first PMCTS rate in Equation~\ref{eq:pmcts_simple_variance}.

After retrospective correction, the factors $d_{i,N}(\tau)$ may depend on all these averages, but are uniformly bounded over paths and $N$. Multiplying the coordinate errors by them therefore retains the $O_{\mathbb P}(N^{-1/2})$ order without an independence assumption for the factors. The population denominator is bounded below with probability tending to one by the deepest-node argument in the proof of Theorem~\ref{thm:pmcts_target_eval_formal}. On this event the empirical denominator is also bounded below with probability tending to one. The difference of the empirical and population ratios is consequently $O_{\mathbb P}(N^{-1/2})$, proving the post-retrospective rate.
\end{proof}

\subsection{Distinct evaluations with DNNs}
\label{app:pmcts_distinct_eval}

\begin{theorem}[Benefit of distinct evaluations (formal)]
\label{thm:pmcts_duplicate_eval_formal}
Fix a pre-existing node $s_t$ visited by particles of both algorithms and condition on their pre-batch histories, sampled trajectories, and pre-retrospective particle weights. For $k\in\{\mathrm S,\mathrm P\}$, let $\mathcal U_k(s_t)$ be the distinct trajectories reaching $s_t$, and let $b_u^k(s_t)$ be the sum of the weights of particles following $u$, normalized across particles reaching $s_t$. Simple PMCTS has unit weights; PMCTS has the importance weights from Equation~\ref{eq:weight_update}. Define
\begin{align}
\label{eq:pmcts_unique_ess}
\mathrm{ESS}_{\mathrm{unique}}^k(s_t)
&=\left(\sum_{u\in\mathcal U_k(s_t)}b_u^k(s_t)^2\right)^{-1}.
\end{align}
Suppose duplicate trajectories share the same backpropagated evaluation error $\epsilon_u^k$; across distinct trajectories these conditional errors have mean zero and are uncorrelated. If their conditional variances $\sigma_{u,k}^2$ are bounded by $\sigma^2$, then the conditional variance due to evaluation error in the pre-retrospective backup satisfies
\begin{align}
\label{eq:pmcts_unique_var}
\mathbb V[\nu_{i,\mathrm{pre}}^k(s_t)\mid\text{trajectories, weights}]
&=\sum_{u\in\mathcal U_k(s_t)}b_u^k(s_t)^2\sigma_{u,k}^2
\leq\frac{\sigma^2}{\mathrm{ESS}_{\mathrm{unique}}^k(s_t)}.
\end{align}
Here $\nu_{i,\mathrm{pre}}^{\mathrm S}=\nu_i^{\mathrm S}$. If all the conditional error variances equal the same $\sigma^2>0$ in both algorithms, equality holds, and PMCTS has lower conditional evaluation-noise variance exactly when $\mathrm{ESS}_{\mathrm{unique}}^{\mathrm P}(s_t)>\mathrm{ESS}_{\mathrm{unique}}^{\mathrm S}(s_t)$. The errors may have nonzero bias relative to true values. Merging identical trajectories and summing their weights leaves a batch's weighted backup unchanged, both before and after its terminal-action retrospective correction.

The following comparison also covers PMCTS's implemented retrospective correction, even when its weights depend on the same evaluation errors. Let $\mathcal F$ denote the conditioned pre-batch histories, trajectories, and pre-retrospective weights. Write $\mathbf b^{\mathrm P,\mathrm{ret}}$ for the vector of normalized, retrospectively corrected group weights at $s_t$. Let $\mathbf b^{\mathrm P,0}$ be the weights obtained by repeating the same correction with every new leaf evaluation in the batch replaced by its conditional mean, while holding the trajectories and pre-retrospective weights fixed. Assume both corrected normalizers are positive, as with the full-support regularized policy of Equation~\ref{eq:pi_i}. Define
\begin{equation}
\label{eq:pmcts_ret_adaptation}
\begin{aligned}
\mathrm{ESS}_{\mathrm{unique,ret}}^{\mathrm P,0}(s_t)
&=\left(\sum_u(b_u^{\mathrm P,0}(s_t))^2\right)^{-1},\\
\delta_{\mathrm{ret}}(s_t)
&=\left(\mathbb E\!\left[
\|\mathbf b^{\mathrm P,\mathrm{ret}}-\mathbf b^{\mathrm P,0}\|_2^2
\|\boldsymbol\epsilon^{\mathrm P}\|_2^2
\mid\mathcal F\right]\right)^{1/2}.
\end{aligned}
\end{equation}
Let $\Delta_{\mathrm{ret}}^{\mathrm P}(s_t)=\sum_u b_u^{\mathrm P,\mathrm{ret}}(s_t)\epsilon_u^{\mathrm P}$ denote the final backup's evaluation-error contribution. Under the common-variance assumption $\sigma_{u,k}^2=\sigma^2>0$, its conditional root mean squared error satisfies
\begin{align}
\label{eq:pmcts_ret_error_bound}
\left(\mathbb E[(\Delta_{\mathrm{ret}}^{\mathrm P}(s_t))^2\mid\mathcal F]\right)^{1/2}
&\leq
\frac{\sigma}{\sqrt{\mathrm{ESS}_{\mathrm{unique,ret}}^{\mathrm P,0}(s_t)}}
+\delta_{\mathrm{ret}}(s_t).
\end{align}
For Simple PMCTS, Equation~\ref{eq:pmcts_unique_var} gives its conditional evaluation-error MSE. PMCTS has a smaller contribution after retrospective correction whenever
\begin{align}
\label{eq:pmcts_ret_sufficient}
\frac{\sigma}{\sqrt{\mathrm{ESS}_{\mathrm{unique,ret}}^{\mathrm P,0}(s_t)}}
+\delta_{\mathrm{ret}}(s_t)
&<
\frac{\sigma}{\sqrt{\mathrm{ESS}_{\mathrm{unique}}^{\mathrm S}(s_t)}}.
\end{align}
This condition also implies lower conditional evaluation-noise variance, since variance cannot exceed mean squared error. The term $\delta_{\mathrm{ret}}$ accounts for using the same evaluations to determine retrospective weights and to compute the backup; it vanishes if the corrected weights do not change with those errors.
\end{theorem}

\begin{proof}
For each distinct trajectory $u$, write its return from $s_t$ as its conditional mean plus $\epsilon_u^k$. Conditional on the paths and pre-retrospective weights, the error component of the backup is $\sum_u b_u^k(s_t)\epsilon_u^k$. Uncorrelatedness removes its cross terms, giving the equality in Equation~\ref{eq:pmcts_unique_var}. The bound follows from Equation~\ref{eq:pmcts_unique_ess}; common positive error variance yields equality and the stated comparison. This comparison is conditional on the trajectories and their weights, and concerns the pre-retrospective backup: the post-expansion weights may themselves depend on the new evaluation errors.

For a group $G_u$ of identical trajectories with the same return $\nu_u(s_t)$, its contribution to a weighted numerator is $\sum_{n\in G_u}w_n\nu_u(s_t)=(\sum_{n\in G_u}w_n)\nu_u(s_t)$; its denominator contribution is $\sum_{n\in G_u}w_n$. Grouping therefore leaves the normalized backup unchanged. Identical trajectories receive the same terminal-action correction $d_{i,N}(u)$, so the same argument applies after retrospective reweighting.

For the final retrospective backup, decompose its evaluation-error contribution as
\begin{align}
\label{eq:pmcts_ret_decomp}
\Delta_{\mathrm{ret}}^{\mathrm P}(s_t)
&=\sum_u b_u^{\mathrm P,0}(s_t)\epsilon_u^{\mathrm P}
+\sum_u\bigl(b_u^{\mathrm P,\mathrm{ret}}(s_t)-b_u^{\mathrm P,0}(s_t)\bigr)
\epsilon_u^{\mathrm P}.
\end{align}
Conditional on $\mathcal F$, the first term has root mean square $\sigma/\sqrt{\mathrm{ESS}_{\mathrm{unique,ret}}^{\mathrm P,0}(s_t)}$ because its weights are fixed and the distinct errors are uncorrelated. Cauchy--Schwarz bounds the absolute value of the second term by $\|\mathbf b^{\mathrm P,\mathrm{ret}}-\mathbf b^{\mathrm P,0}\|_2\|\boldsymbol\epsilon^{\mathrm P}\|_2$, whose conditional root mean square is $\delta_{\mathrm{ret}}(s_t)$. The triangle inequality in conditional $L^2$ proves Equation~\ref{eq:pmcts_ret_error_bound}, with no independence assumption between the final weights and errors. Simple PMCTS's weights are fixed given $\mathcal F$, so its error contribution has conditional MSE $\sigma^2/\mathrm{ESS}_{\mathrm{unique}}^{\mathrm S}(s_t)$. Equation~\ref{eq:pmcts_ret_sufficient} therefore gives the stated strict comparison.
\end{proof}


\section{Worst-Case Complexity Analysis}
\label{app:complexity_analysis}
In this section we analyze the computation complexity of Simple PMCTS (Algorithm \ref{alg:simple_pmcts}) and PMCTS (Algorithm \ref{alg:pmcts}) under a model with \(O(MN)\) parallel workers and a batched model evaluation treated as a single parallel operation.

We begin with a brief complexity analysis of baseline MCTS, for reference.

\subsection{MCTS Worst Case Complexity Analysis}
MCTS is given a search budget of $M$ \textit{simulations}, i.e. iterations of selection, expansion, and backpropagation. Expansion takes a constant amount of time, and both selection and backpropagation scale linearly with the depth $d_i$ of the newly expanded leaf. $d_i \leq M$ since in each iteration, the depth of the tree increases by at most one. Thus we can bound the worst-case runtime complexity by $\O(M^2)$.

Each expansion step adds a new node to the tree, and the rest of the operations (selection and backup) use a constant amount of space, so in total the space complexity of MCTS is $\O(M)$.

\subsection{Simple PMCTS Worst Case Complexity}
We analyze the worst-case complexity of Simple PMCTS (and next PMCTS) with respect to \(M\) parallel search iterations and \(N\) particles per iteration,
corresponding to a tree size of at most \(MN\).

\textbf{Selection:} Each particle can do selection completely independently, so the time taken again scales linearly in the maximum depth across particles $d_i$. Even though we may add up to $N$ nodes in each iteration, the depth still only increases by at most one, so we have $d_i \le M$ again.

\textbf{Expansion:} Expansion is done as a batch. We compute the value predictions, new prior policies, and add up to $N$ nodes to the tree. This can be done in constant time by having every particle create the new node in a predetermined spot based on its index. To connect the parent nodes to the new children, each particle tries to assign its child index to the parent, and in the case of contention, one arbitrarily wins out. Each particle then checks what was written and reassigns itself to the child node that won.

\textbf{Backpropagation:} Each particle can independently calculate its contribution to each node by tracing its trajectory in the tree. This will take a linear time in $d_i$. The contributions can then be combined for each state in parallel, either with a linear scan $\O(N)$ or tree reduction $\O(\log(N))$ \cite{tree-reduction}. Thus the time for backpropagation scales as $\O(d_i + \log(N))$.

\textbf{Total:} In each iteration we take $\O(d_i)$ for selection, $\O(1)$ time for expansion, and $\O(d_i + \log(N))$ for backpropagation. Over $M$ iterations this amounts to $\O(M (M + \log(N))) = \O(M^2 + M \log(N))$.

\textbf{Space:} Since each iteration can add at most $N$ new nodes, the space required for the search tree is $O(M N)$. Additional space is required during backpropagation to keep track of the contribution of each particle at each affected node: There are $N$ particles and each affected at most $d_i$ nodes, so $\O(d_i N) \leq \O(M N)$ space.

\subsection{PMCTS Worst Case Complexity}
\label{app:pmcts_complexity}

\textbf{Selection:} Selection is much the same as in Simple PMCTS. The few differences are that we must now compute a sampling policy $\hat{\pi_i}$ to sample from, and the particle weights must be updated. The weight update is constant-time per step, and when using a temperature-based sampling policy as in this work, $\hat{\pi_i}$ can also be computed in constant time (independently by each particle). This gives the same linear time $\O(d_i) \leq \O(M)$ for selection.

\textbf{Expansion} remains identical to Simple PMCTS, so $\O(1)$.

\textbf{Backpropagation:} Most of PMCTS's features come into play in this step. We compute the unweighted returns contributed by each particle in parallel in $O(d_i)$ time.
\textit{Retrospective reweighing} can be done in parallel for each particle in constant time.
De-duplication of particles can be done using \textit{reduce-by-key} in $\O(\log(N))$ time \cite{fast-deduplication} (i.e. by using a parallel sort, detecting boundaries, and then tree-reducing).
Computing ESS can be done for each affected state in parallel in time $\O(\log(N))$ using tree reduction for sums over particle weights, and parallelism across particles for operations such as squaring or dividing weights.
Calculating the weighted return similarly takes $\O(\log(N))$ time with tree reduction and can be done in parallel per state.
In total the worst-case time complexity for backup is $\O(d_i + \log(N))$.

\textbf{Total:} Similar to Simple PMCTS, the worst-case complexity of one iteration is $\O(M + \log(N))$ and $\O(M^2 + M\log(N))$ the entire algorithm.

\textbf{Space:} The space required for the search tree is $\O(MN)$. Space required for backpropagation is dominated by the weighted return and ESS calculations which need $\O(d_i N) \leq \O(M N)$ as before to store the contribution of each particle.

%% file: appendix/additional_prior_work.tex
\section{MCTS and Parallelization: A Brief Survey}
\label{app:more_related_work}

\subsection{Synchronous vs. asynchronous parallelism}
\label{app:sync_vs_async}
A system of parallel processes is called \textit{synchronous} if all processes run using the same clock, i.e., all operations are executed in lockstep (physically or conceptually).
It is called \textit{asynchronous} if each process has
its own independent clock, i.e. the different processes execute independently, potentially interleaving arbitrarily \cite{sync_vs_async}.

Asynchronous parallelization is very popular, as it allows multiple processes to run independently in parallel, and is indeed a popular way to implement virtual-visits based heuristic parallelization of MCTS \cite{analysis_virtuals}.

In contrast, modern GPU acceleration with JAX \cite{jax2018github} enables a form of semantically synchronous parallelization, typically expressed via batching or vectorization. 
In this setting, computations are defined so that elements of a batch are processed independently but conceptually in lockstep, with no observable ordering or interaction between them.
This has popularized synchronous parallelized acceleration of the RL framework, based on ``batched" operations, for training as well as search \cite{trtpi,tsmcts,gumbel_mz}.
As a result, it is much easier to benefit from GPU acceleration with Jax with conceptually synchronous parallelization.

\subsection{Parallel MCTS: A brief survey}
\label{app:more_parallel_mcts}
Many approaches for parallelizing MCTS under different practical constraints have been proposed over the years \cite{parallel_uct,virtual_loss,lock_free_parallel_mcts_1,parallel_mcts_2,lock_free_parallel_mcts_2,batch_mcts,malmsten2025transzero}.
Among them, the three canonical approaches are 
\textit{root parallelization}, \textit{leaf parallelization} and \textit{tree parallelization} \cite{parallel_uct,virtual_loss}. 

\textbf{Root parallelization} parallelizes across multiple instances of MCTS.
To achieve runtime inference scaling, all instances run from the same root state and aggregate their results once search has concluded.
This approach builds search trees of the same size, and in expectation each tree is the same.
It also relies heavily on stochasticity in the selection and/or evaluation and thus is not well suited to modern uses of MCTS that rely on fully deterministic selection and evaluation.

A variant of this approach is very effective for scaling up during training however.
Instead of running multiple instances of MCTS in the same state in parallel and aggregating their results, the search can be parallelized across $N$ \textit{different} states.
The results are then aggregated in a training batch over these states.
This approach is the go-to in JAX-based implementations of training RL agents with MCTS for policy improvement (e.g. AlphaZero variants) \cite{mctx}. 

\textbf{Leaf parallelization} parallelizes the evaluation of the same leaf across $N$ parallel processes. 
This is only effective when evaluation is stochastic, and thus is not suited for modern algorithms either.

\textbf{Tree parallelization} parallelizes the search with multiple parallel, synchronous or asynchronous, \textit{selection, expansion} and \textit{backpropagation} processes operating on the same tree. 
The immediate challenge these methods run into are duplicate selection trajectories.
To overcome this challenge, the canonical approach in practice is the \textit{virtual visits} family of heuristics.
To reduce the probability that different processes attempt to expand the same leaf, the main idea is that when a process arrives at a node during selection, it appends a count of "virtual" visits \textit{before} the expansion and backpropagation steps complete.
These visits are taken into account by other processes arriving at the same node, \textit{as if} the node was fully updated with backpropagation.

This approach is very popular in practice, and has been very successful for benefiting from parallel compute during inference time.
However, it is not principled, in the sense that there are no guarantees over which policy is evaluated by the tree at each iteration.

Additional approaches have iterated on the notion of virtual visits with \textit{informed}-virtual visits, using smaller DNNs or other forms of cheaper evaluations to direct the search, while spending the cost for expensive evaluations only later or only on specific nodes \citep{speculate_parallel_mcts,speculative_tree_traversal,fast_slow_nets_pmcts}.

Prior work analyzed this direction theoretically from the perspective of \textit{how much worse than regular MCTS can the decisions induced by virtual visits be} \citep{on_effective_parallelization_of_mcts}.
In contrast, PMCTS recovers the expected-backpropagation limit of MCTS as \(N\to\infty\) and uses parallel sampling to reduce Monte Carlo error under the classical assumptions.

Overall, a large number of approaches have investigated using virtual visits in different implementations with different levels of synchronous and asynchronous parallelization \citep{virtual_loss,parallel_uct,lock_free_parallel_mcts_1,parallel_mcts_2,lock_free_parallel_mcts_2,batch_mcts}.

\textbf{Other approaches} for parallelizing MCTS have been proposed. \cite{malmsten2025transzero} replace the DNN-learned dynamics model of MuZero \cite{muzero} with a transformer, to unroll multiple action trajectories in the tree in a batch.
This approach is not suited for MCTS in general, only for DNN-based dynamics models, modifies the backpropagation of MCTS, and does not provide any theoretical analysis or guarantees of policy improvement with respect to the modified selection policies.

%% file: appendix/additional_results.tex
\section{Additional Results}
\label{app:additional_results}

Scaling across $M$ in an example environment is investigated in Figure \ref{fig:scaling_and_sims_gardner}.
\begin{figure}[H]
    \vspace{-0.3cm}
    \centering
    \includegraphics[width=1\linewidth]{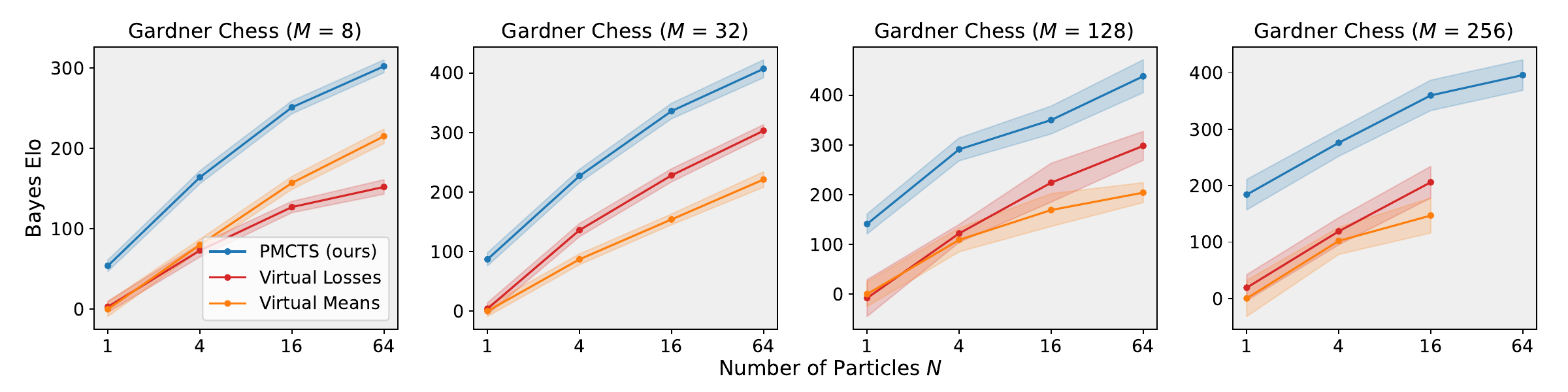}
    \vspace{-0.65cm}
    \caption{
    Scaling of parallel MCTS variants with parallel compute (number of particles $N$) and simulation budgets $M$ in Gardner chess. Bayes Elo with 95\% CI.
    }
    \label{fig:scaling_and_sims_gardner}
    \vspace{-0.4cm}
\end{figure}

The individual boardgames results aggregated in Figure \ref{fig:scaling} are presented in Figure \ref{fig:indiv_boardgame_results}.

\begin{figure}[H]
    \vspace{-0.3cm}
    \centering
    \includegraphics[width=1\linewidth]{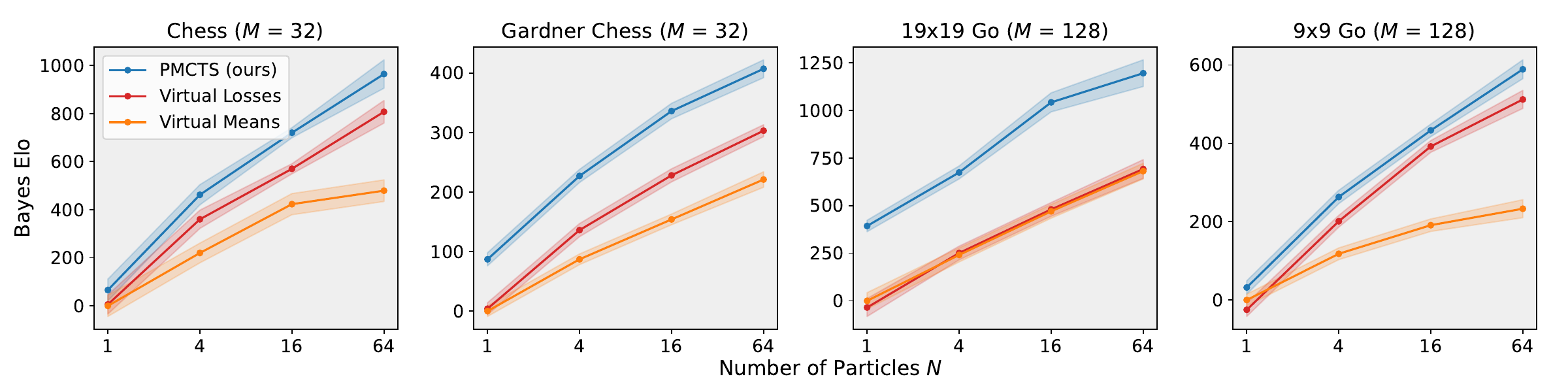}
    \vspace{-0.65cm}
    \caption{
    Scaling of tree parallel MCTS variants with parallel compute (number of particles $N$) in different boardgames. Bayes Elo with 95\% CI.
    }
    \label{fig:indiv_boardgame_results}
    \vspace{-0.4cm}
\end{figure}


The individual discrete and continuous action results aggregated in Figure \ref{fig:scaling} are presented in Figure \ref{fig:all_smz_envs}.

\begin{figure}[H]
    \vspace{-0.3cm}
    \centering
    \includegraphics[width=1\linewidth]{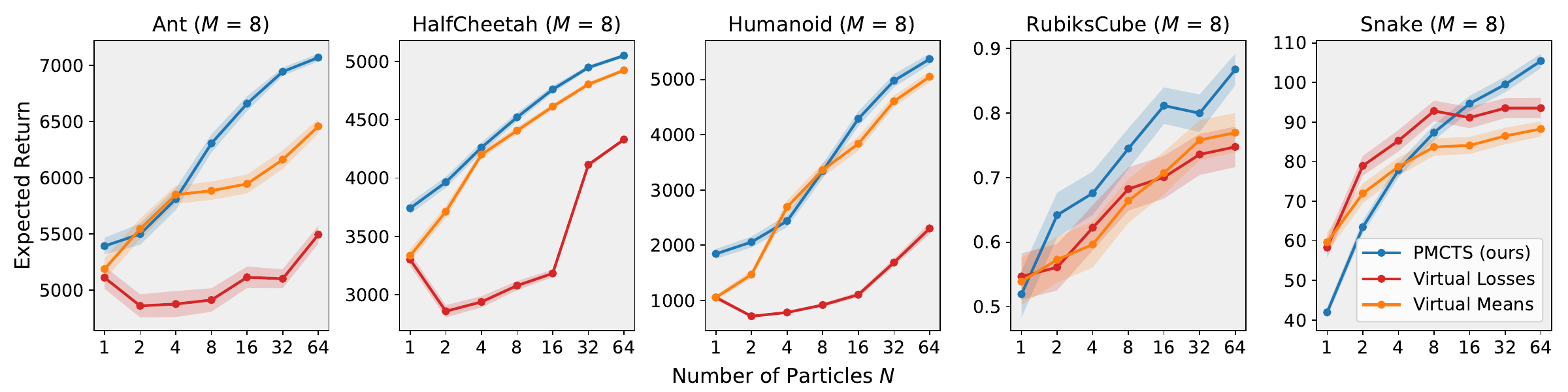}
    \vspace{-0.65cm}
    \caption{
    Scaling of parallel MCTS variants with parallel compute (number of particles $N$) across different environments in the discrete and continuous domains. 
    Mean expected return with $ 95\%$ Gaussian CI.
    }
    \label{fig:all_smz_envs}
    \vspace{-0.4cm}
\end{figure}

A comparison between tree-parallel MCTS with PMCTS and root-parallel MCTS is presented in Figure \ref{fig:pmcts_vs_gmcts}, 
showing that, in the evaluated deterministic-DNN settings, root-parallel GumbelMCTS does not scale inference at runtime effectively with additional parallel compute.
For this comparison, we use a modern version of non-parallel MCTS, GumbelMCTS, and run it vectorized across different instances at the same state, which also allows us to compare the performance of PMCTS to a popular version of non-parallel MCTS.

\begin{figure}[H]
    \vspace{-0.3cm}
    \centering
    \includegraphics[width=1\linewidth]{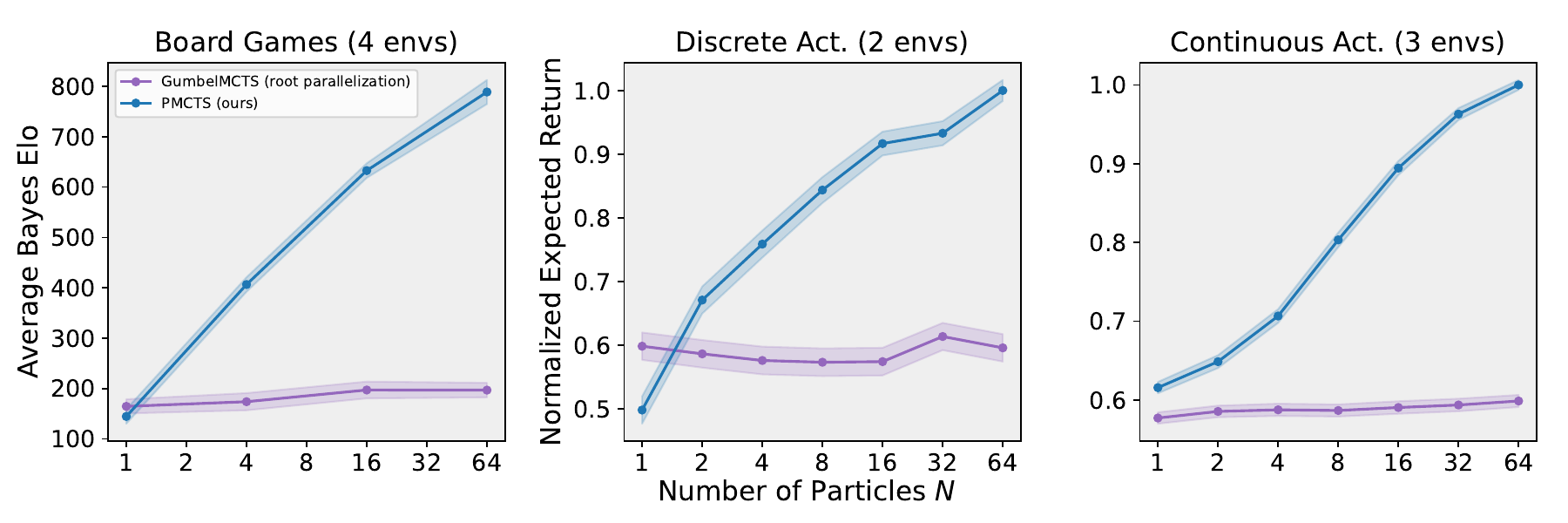}
    \vspace{-0.65cm}
    \caption{
    Scaling of tree parallel MCTS with PMCTS vs. root parallel MCTS with GumbelMCTS across boardgames, discrete and continuous environments. 
    Average Bayes Elo or expected return with $ 95\%$ CI.
    }
    \label{fig:pmcts_vs_gmcts}
    \vspace{-0.4cm}
\end{figure}

In Figure \ref{fig:action_selection_ablation} we investigate the effect of different action selection mechanisms across the different baselines.
In this representative 9x9 Go setting, the action-selection mechanism used by each baseline in all other experiments is the strongest among the alternatives evaluated for that baseline.
\textit{Improved Policy} refers to selecting the action that maximizes the improved policy. 
\textit{Max. Visits} chooses the action with the most visits. For GumbelMCTS, we compare different ways to aggregate across trees, specifically averaging the improved policy (\textit{Aggregate Policy}), averaging the Q-values (\textit{Aggregate Values}), and choosing the action picked by the most independent searches (\textit{Aggregate Voting}), see Appendices \ref{app:acting_in_eval} and \ref{app:action_selection_and_improvement_survey} for more detail.

\begin{figure}[H]
    \vspace{-0.3cm}
    \centering
    \includegraphics[width=1\linewidth]{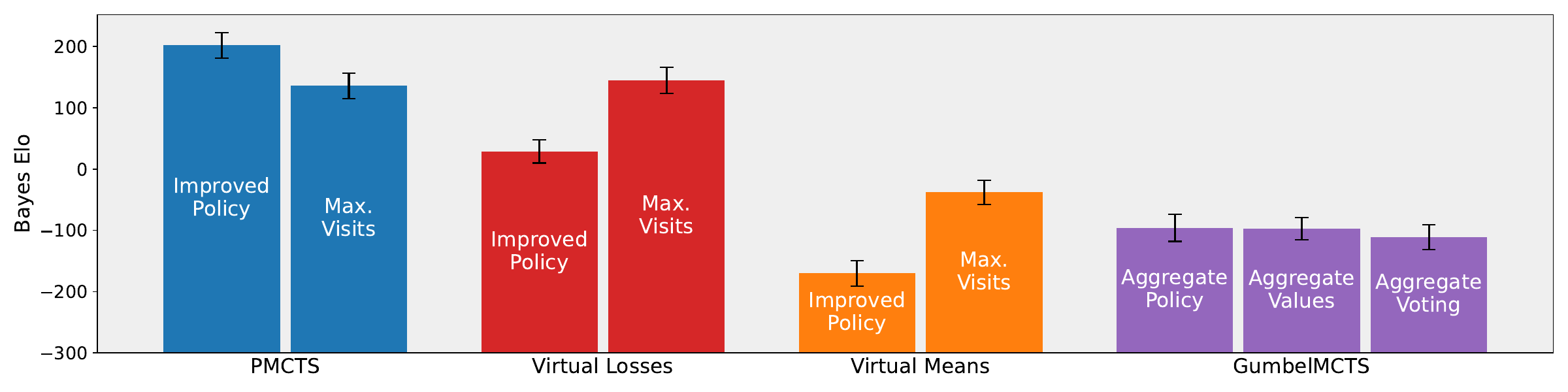}
    \vspace{-0.65cm}
    \caption{
    Action selection ablations across the different baselines, in 9x9 Go ($M=128,N=16$).
    Bayes Elo with $ 95\%$ CI.
    }
    \label{fig:action_selection_ablation}
    \vspace{-0.4cm}
\end{figure}

We include a comparison to TSMCTS and SMC on Gardner chess and 9x9 Go, demonstrating their challenge to perform well with a larger simulation budgets (and smaller particle budgets, in the case of TSMCTS), in Figure \ref{fig:smc_tsmcts_comparison}.

\begin{figure}[H]
    \vspace{-0.3cm}
    \centering
    \includegraphics[width=1\linewidth]{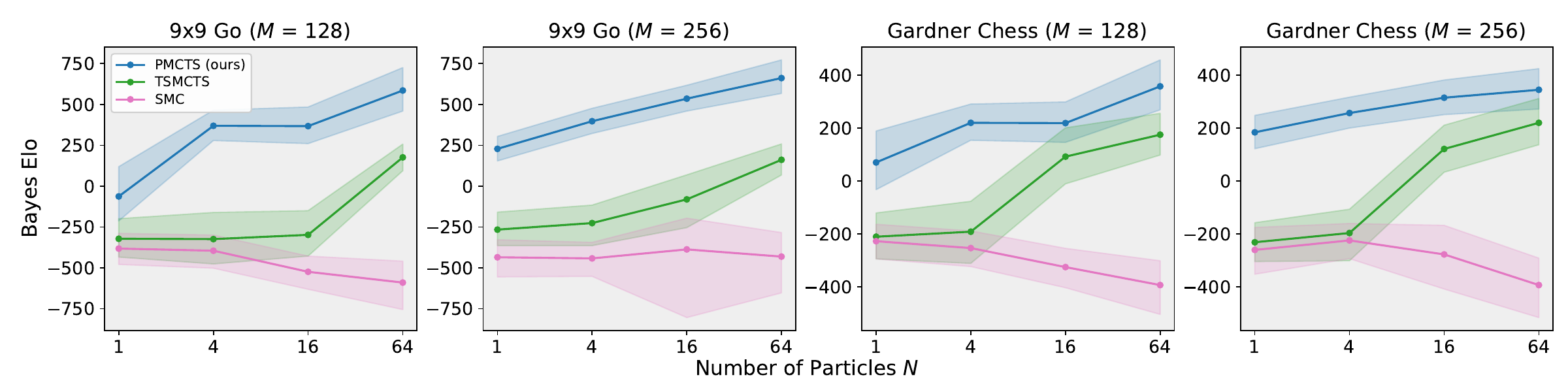}
    \vspace{-0.65cm}
    \caption{
    PMCTS vs. TSMCTS and baseline SMC in Gardner chess and 9x9 Go, with larger simulation budgets $M = 128, 256.$
    Bayes Elo with $ 95\%$ CI.
    }
    \label{fig:smc_tsmcts_comparison}
    \vspace{-0.4cm}
\end{figure}

To investigate how the temperature \(\eta\) trades off trajectory diversity against divergence from the target \(\pi_i\), Figure~\ref{fig:statistics} reports search statistics averaged across trees rooted at different states in 9x9 Go, with \(M=128\) and \(N=64\) unless otherwise stated.
The left panel shows that the number of unique trajectories increases with \(\eta\). 
The second panel shows the corresponding increase in \(\mathrm{KL}(\hat{\pi}_i\Vert\pi_i)\). 
The third panel reports \(1/\mathrm{ESS}\), the estimator of the variance.
Among the evaluated temperatures, \(\eta=1.5\) consistently gives the lowest value variance estimate and the right panel shows that the same temperature also achieves the best performance, which supports the inverse ESS as a variance estimator on the one hand, and demonstrates how aiming to minimize the variance of the update results in improved performance on the other hand.

\begin{figure}[H]
    \vspace{-0.3cm}
    \centering
    \includegraphics[width=1\linewidth]{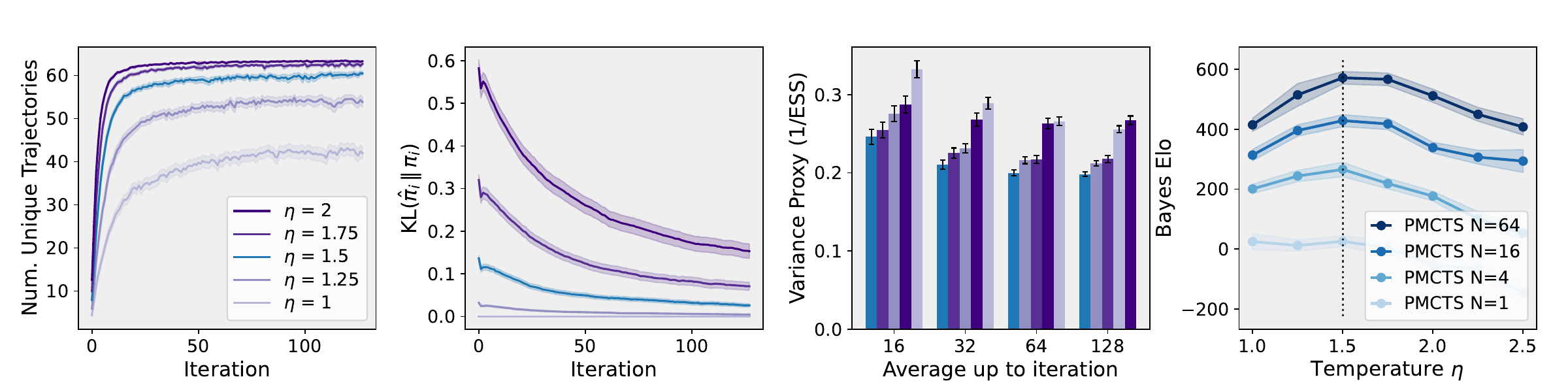}
    \vspace{-0.65cm}
    \caption{
    Search statistics for PMCTS on 9x9 Go with $M=128$. Means and 95\% Gaussian or Bayes Elo CIs, where relevant.
    Among the evaluated temperatures, \(\eta=1.5\) gives the lowest estimated variance (second from right) and the highest observed performance (right).
    }
    \label{fig:statistics}
    \vspace{-0.4cm}
\end{figure}

%% file: appendix/implementation_details.tex
\section{Implementation Details}
\label{app:imp_details}

\subsection{The search policies PUCT and UCT}
\label{app:puct_and_uct}
\begin{align}
    \pi_{i}^{\text{UCT}}(s) &:= 
    \argmax_{a \in A}
    q_{i}(s_t, a) + 
    C 
    \sqrt{\frac{ \log M_i(s_t)}{1 + M_i(s_t, a)}}.
    \\
    \pi_{i}^{\text{PUCT}}(s) &:= 
    \argmax_{a \in A}
    q_{i}(s_t, a) + \pi_\theta(a | s_t) \, 
    \left(c_{\mathrm{init}}
+
\log\!\left(
\frac{M_i(s)+c_{\mathrm{base}}+1}
     {c_{\mathrm{base}}}
\right)
\right)
    \frac{ \sqrt{M_i(s_t)}}{1 + M_i(s_t, a)}.
\end{align}

\subsection{Improved Policy, Value and Action Selection - A Brief Survey of Design Choices}
\label{app:action_selection_and_improvement_survey}
At the end of search, P/MCTS returns an action, a policy target $\pi_{\mathrm{search}}$ for $\pi_\theta$, and a value target $V_{\mathrm{search}}$ for $v_\phi$. 
In this appendix $q_M$ denotes the post-search values after $M$ simulations; Algorithms~\ref{alg:simple_pmcts} and~\ref{alg:pmcts} denote those same values by $q_{M+1}$ because their loops start at $i=1$.
This can be (and is) done in multiple different ways, which we discuss below.

\textbf{MCTS for Heuristic Policy Improvement:}
In AlphaZero \citep{AlphaGo,AlphaZero}, MCTS is used to train the policy.
As prior to \cite{grill2020monte}'s work, MCTS was not theoretically interpretable as a policy improvement operator, it was not obvious what target to extract for the prior policy.
\cite{AlphaGo} chose the visitations distribution $M(s_1,a)$:
\begin{align}
    \pi^{\text{classical}}_{\text{search}}(a|s_1) = \frac{M(s_1,a)}{\sum_{b \in \A} M(s_1,b) }
\end{align}
Although this does not guarantee policy improvement \citep{gumbel_mz}, it has been extremely successful in practice and paved the way for the great influence MCTS-based agents had on RL.

It is natural to then compute the root value and select an action to take as:
\begin{align}
    V_{\text{search}}(s_1) = \sum_{a \in \A} \pi^{\text{classical}}_{\text{search}}(a|s_1)q_{M}(s_1,a), \quad a \sim \pi^{\text{classical}}_{\text{search}}(s_1).
\end{align}
When we want a deterministic action, we can select the action with the most visits $M(s_1, a)$.

\textbf{MCTS as a Policy Improvement Operator}
\cite{grill2020monte} showed that MCTS \textit{tracks}, \textit{approximates} and therefore \textit{can be interpreted as} a policy improvement operator.
More specifically, their results can be summarized as follows:
(I) The popular selection policies UCT and PUCT \textit{converge over iterations} (i.e., track) to the solution of an optimization problem of the form:
\begin{align}
    \pi'(s) = \argmax_{\pi'} \sum_{a \in \A} \pi'(a|s)q_i(s,a) - \lambda(M_i(s)) D(\pi',\pi_\theta).
    \label{eq:grills_finding_optimization_problem}
\end{align}
Here, $D$ is some regularization term and $\lambda(M_i(s,a))$ is a function of the visitation counts $M_i(s)$. 
Further, \cite{grill2020monte} showed that PUCT corresponds to the specific regularization term $D(\pi',\pi_\theta) = KL(\pi_\theta,\pi')$, i.e. the reverse KL divergence.
We apply the following definition of the $KL(x,y)= \sum_{a \in \A}x(a)\log \frac{x(a)}{y(a)}$.

Next, \cite{grill2020monte} showed that the solution $\pi'$ (Equation \ref{eq:grills_finding_optimization_problem}) can be used directly, for three purposes:
First, for selection actions in the environment. Second, for training the policy. And third, for driving the search itself, replacing deterministic PUCT with stochastic $\pi'$.

The operator $\pi' = \I(\pi_\theta,q_M)$ which solves Equation \ref{eq:grills_finding_optimization_problem} is not a very popular operator, and requires solving an optimization problem numerically to find the normalization constant of $\pi'$.
\cite{gumbel_mz} showed that it can be replaced by $\I_{\text{reg}}$ effectively.

Accordingly, this suggests the following natural return values:
\begin{align}
    \pi_{\text{search}}(a|s_1) = \I(\pi_\theta, q_M), \,\, V_{\text{search}}(s_1) = \sum_{a \in \A} \pi_{\text{search}}(a|s_1)q_{M}(s_1,a), \,\, a \sim \pi_{\text{search}}(s_1).
\end{align}

\textbf{Sequential Halving and Completed Qs:}
\cite{gumbel_mz} introduced \textit{Sequential Halving} (SH) to search at the root.
SH is a bandit algorithm which optimizes for \textit{simple-regret, best action identification} given \textit{known} search budget $M$.
In contrast, PUCT was designed for \textit{anytime, cumulative-regret, best action identification}.
SH can be interpreted in two different ways:
First, it is just a budget-allocation algorithm, which assigns different actions different action budgets, without a direct influence on the action values themselves. 
One qualification is that the backed-up values are nonstationary: later simulations evaluate different selection policies as the tree is updated. \citet{UCT} model this nonstationarity as drift. 
Root budget allocation changes which actions are evaluated and therefore the resulting search values $q_i$; each expected backup still evaluates the selection policy specified for that batch under the assumptions of Appendix~\ref{app:pmcts_theory}.

Second, SH can be viewed as a \textit{best-action selection} algorithm. 
At the last iteration of SH, it outputs a recommended-best action.
\cite{gumbel_mz} used SH to select actions in the environment.
However, they found that in order to retain sufficient exploration in board games, the agents benefited from sampling proportional to $\pi^{classical}_{search}$, the policy induced directly by the visitation counts $M(s,a)$.

Additionally, \cite{gumbel_mz} proposed the completed Q values, which predicts the value $q(s,a)$ of unvisited actions $M(s,a) = 0$ with a mixture of the form:
\begin{align}
    \label{eq:completed_qs}
    q_i(s,a) &= 
    \begin{cases} 
        q_i(s,a) & \text{if} M(s,a) > 0 \\  
        v_{\text{mix}}(s) & \text{if} M(s,a) = 0 
    \end{cases},
    \\
    v_{\text{mix}}(s) &= \frac{1}{1 + \sum_a M(s, a)} \left( v_\phi(s) + \frac{\sum_a M(s, a)}{\sum_{a \in A} \pi_\theta(a|s)} \sum_{a \in A} \pi_\theta(a|s) q_i(s, a) \right),
    \\
    A &= \{a : M(s, a) > 0\}.
\end{align}
The value $v_{\text{mix}}(s)$ is intended as a more stable estimator of $V^{\pi_\theta}(s)$ than the single DNN prediction $v_{\phi}(s)$.
\cite{gumbel_mz} showed that using estimators of this form preserves policy improvement.

This distills to the following return values:
\begin{align}
    \pi^{\text{GMZ}}_{\text{search}}(s_1) &= \I_{\text{reg}}(\pi_\theta, q_M)(s_1), \quad V_{\text{search}}(s_1) = \sum_{a \in \A} \pi^{\text{GMZ}}_{\text{search}}(a|s_1)q_M(s_1,a), \\
    a &\sim \pi^{\text{classical}}_{\text{search}}(s_1), \quad \text{or} \quad a = \text{SH recommended action}.
\end{align}

\textbf{Putting it all together}
Recent work from the SMC for search literature \citep{tsmcts} uses SH only for search budget allocation at the root $s_1$ to great benefit. 
Using SH only for root budget allocation changes which values are observed, while the returned policy is still computed by applying $\I$ to the resulting search values $q_M$.

However, the completed Qs, although sound for policy improvement, may result in large mass for unvisited actions.
This results in potentially a very high entropy policy for action selection in the environment, which may be catastrophic in environments where one bad move can terminate the agent (falling off of cliffs, or moves leading to immediate loss in boardgames).

Putting all of this together, we opt for the following:
\begin{align}
    \label{eq:actual_search_return_scores}
    \pi_{\text{search}}(a|s_1) &= \I(\pi_\theta, q_M), \quad V_{\text{search}}(s_1) = \sum_{a \in \A} \bar \pi(a|s_1)q_M(s_1,a), \quad a \sim \bar \pi(s_1). \\
    \bar \pi(a|s_1) &\propto 
        \begin{cases}
            \pi_{\text{search}}(a|s_1) \quad \text{if:} \quad  M(s_1,a) > 0 \\
            0 \quad \quad \quad \quad \quad \quad  \text{if:} \quad  M(s_1,a) = 0
        \end{cases}
\end{align}
The policy $\bar \pi$ has mass only on visited actions, and otherwise is proportional to $\pi_{\text{search}}$.
As a result, the root value $V_{\text{search}}(s_1)$ only incorporates $q_M$ of visited actions (e.g. no completed Q values participate in the value computation).
Similarly, the acting policy is not over-entropic and only incorporates information from visited actions (e.g. again, no completedQ values participate in the value computation).
$\bar \pi$ and can be thought of as a more principled way to achieve the sampling used by GumbelMCTS $ a \sim \pi^{\text{classical}}_{\text{search}}(s_1) $.

Using $V_{\text{search}}$ to generate value targets \citep{oren2025epistemic} has been shown to be sound by \citet{viac}, despite $V_{\text{search}}$ not being the value of the policy $\pi_\theta$.

\subsection{The Virtual Losses and Virtual Means Heuristics}
\label{app:virtuals}
We use two popular variations of \textit{virtual visits}: Virtual Means (VM) \cite{batch_mcts} and Virtual Losses (VL) \citep{analysis_virtuals, virtual_loss}.
Under these heuristics, the deterministic search policy is computed as follows:
\begin{align}
    \pi_{\text{VL},i}^{\text{PUCT}}(s) &:= 
    \argmax_{a \in A}
    q^{\text{VL}}_{i}(s_t, a) + \pi_\theta(a | s_t) \, 
    C_i(s_t)
    \frac{ \sqrt{M_i(s_t) + M^{\text{virtual}}_i(s_t)}}{1 + M_i(s_t, a) + M^{\text{virtual}}_i(s_t,a)}.
    \\
    \pi_{\text{VM},i}^{\text{PUCT}}(s) &:= 
    \argmax_{a \in A}
    q_{i}(s_t, a) + \pi_\theta(a | s_t) \, 
    C_i(s_t)
    \frac{ \sqrt{M_i(s_t) + M^{\text{virtual}}_i(s_t)}}{1 + M_i(s_t, a) + M^{\text{virtual}}_i(s_t,a)}.
    \\
    C_i(s_t)
    &:=
    c_{\mathrm{init}}
    +
    \log\!\left(
    \frac{M_i(s_t)+c_{\mathrm{base}}+1}
        {c_{\mathrm{base}}}
    \right).
\end{align}
This is the PUCT formula \cite{puct} but including the virtual visits added by unevaluated particles / processes. $q^{\text{VL}}$ is $q_i$ except each virtual visit is treated as a \textit{loss}, i.e. a return of $-1$.
We keep the implementation of VM and VL consistent with mctx's \cite{mctx} PUCT implementation: the Q values of unvisited actions are completed with the value prediction $v_\phi(s)$ and normalized after virtual losses are added. 
After normalization, the normalized Q-values for unvisited actions are set to 0. 

\subsection{Sequential Halving with P/MCTS}
\label{app:sh_pmcts}
GumbelMCTS \cite{gumbel_mz} modifies the search at the root of the tree with Sequential Halving (SH, \cite{sh}).
The SH algorithm is a budget-allocation algorithm, which takes as input \textit{number of actions to search} $K$ and divides the search into $\log_2 K$ iterations. 
Each iteration is assigned a search budget $ B = M / \log_2 K. $
At each iteration $i \geq 1$, the number of searched actions halves $ K_i = K / 2^{i-1} $, and assigned equal budget: $ B / K_i. $

When implementing SH with particles, the budget is treated as $ M N$ instead of $M$, and divided accordingly. To keep the parallelization across $N$ intact, $N $ particles are used during search at every iteration. 
See \cite{tsmcts} for more detail.

SH interacts with PMCTS: on the one hand, it allows PMCTS to drive diversity by setting large $ K = N$ for example.
On the other hand, in later iterations, all particles are forced to search the same actions, which can reduce trajectory diversity.
As a result, it's not clear whether SH should increase performance or reduce performance when coupled with PMCTS.
In our experiments, we were not able to see benefit from using a parallel, particle-based variation of SH at the root of PMCTS.

\subsection{Baselines}
\label{app:baselines_implementation}
All baselines are implemented within our PMCTS framework written in JAX \cite{jax2018github}.
Virtual Means / Losses are thus also implemented as batch-parallel and synchronous. This means that all ``particles'' (what would be threads or processes) have to select trajectories fully sequentially in order to communicate the virtual losses added. If we attempted to select all actions simultaneously all ``particles'' would pick the same action since Virtual Means / Losses are deterministic.
Their backpropagation is the same as used by PMCTS, but without the PMCTS-specific features such as ESS, de-duplication, etc.
As a result, Virtual Means / Losses only parallelize the expansion well, which is not the best case for these algorithms.
In more standard implementations of Virtual Means / Losses, selection and backup are done asynchronously with atomics or locks to deal with the situation where multiple threads are at the same node at the same time. Even in these implementations however, in the worst case, when \textit{all} threads are at the same node at the same time, the selection is effectively sequential. 
This fine-grained shared-tree asynchronous pattern is not expressible within our fully JIT-compiled batched JAX/XLA implementation without leaving the staged array-programming model or introducing custom synchronization primitives.

GumbelMCTS and PUCT MCTS are implemented in our setup, to keep comparison as direct as possible, based on the official implementation \cite{mctx}. 

\textbf{Root-parallelization with GumbelMCTS} We conduct independent search on $N$ copies of the state with different random keys and then aggregate the results. 
In continuous action environments the action space at the root is shared by sampling $K$ actions from $\pi_\theta$ once and treating this as a discrete action space \cite{sampled_mz}. 
From that point on, every search tree is unique.
Once search has concluded, we aggregate across search trees by taking either the mean of the Q values (in continuous environments) or the mean of the policy (in discrete environments): $\bar q_M(s,a) = \frac{1}{N}\sum_{i = 1}^N q^i_M(s,a)$ or $\bar \pi = \frac{1}{N} \sum_{i=1}^N \pi_M^i$.

%% file: appendix/experimental_details.tex
\section{Experimental Details}
\label{app:experimental_details}

\subsection{Environments and Scaling Experiments Across Domains}
\label{app:scaling_experiments}
We evaluate PMCTS empirically across four domains and 10 environments: four two-player zero-sum board games (chess, Gardner chess, 19x19 Go and 9x9 Go) implemented by PGX~\cite{koyamada2023pgx};
Two standard discrete-control single-agent problems (Snake and RubiksCube) implemented by Jumanji~\cite{bonnet2024jumanji};
Three classical continuous control tasks (Ant, Humanoid and HalfCheetah) implemented by Brax~\cite{brax}.
One LLM-scaling inference in the environment Sokoban \cite{bonnet2024jumanji} described in more detail in Appendix \ref{app:llm_experiment}.



We have DNN models ($v_\phi$ and $\pi_\theta$ pairs) for each environment: In 9x9 Go and Gardner chess we use the baseline models provided by PGX. For full chess and 19x19 Go we trained our own networks as PGX does not supply those (architecture is described in Appendix \ref{app:net_archs}). In the discrete and continuous control environments we use three models (three pairs), trained independently for different numbers of steps. We use checkpoints prior to convergence to keep the evaluation as meaningful as possible.

For each environment, agent, particle budget $N$ and model, we run multiple evaluation episodes.
In the board game environments, this amounted to each variant (agent and particle budget $N$) playing multiple games against each other variant, and computing Bayes Elo, which also provides a 95\% confidence intervals. The games started from positions in an opening book, generation of which is explained in \ref{app:opening_book}.
In the single-player environments, this amounted to each variant (environment, agent, particle budget $N$ and model) running 256 episodes.
We then compute the mean and a $\pm 2$-SEM based $> 95\%$ confidence interval. 
In the single-player environments, every combination of environment,
agent, particle budget \(N\), and model was evaluated for 256 episodes.
For each environment--agent--\(N\) combination, we pool the returns
from all three models and report their mean and the Gaussian interval
\[
\bar R \pm 1.96\,\frac{s_R}{\sqrt{3\cdot256}},
\]
where \(s_R\) is the standard deviation across the pooled
model--episode returns. Thus, these intervals treat model--episode
pairs as the evaluation units.

For the aggregated results in Figure~\ref{fig:scaling}, each return is first normalized within its environment using zero as the minimum observed return and
the maximum observed return across all evaluated agents as the upper
endpoint. 
We then pool the normalized returns across the environments in the corresponding domain and compute the mean and Gaussian interval in the same manner.

\subsection{Opening Book}
\label{app:opening_book}

All agents in the board-game scaling experiments are deterministic. Starting every match from the initial position would make all matches identical, so we use an opening book: a set of starting states. 
Each environment's book contains 20{,}000 unique positions, generated by sampling from the policy network up to a predetermined move count, then filtering for roughly balanced positions. A position is roughly balanced when that model's value prediction lies in $[-0.3, 0.3]$. The number of moves sampled per environment is given in Table~\ref{tab:hps_board_games}.

\subsection{Acting in the Environment}
\label{app:acting_in_eval}
The different agents act as follows:
PMCTS selects the action that maximizes the visited-action acting policy at the root: $\arg\max_a\bar\pi(a\mid s_1)$, with $\bar\pi$ defined in Equation~(51).
Virtual Means / Losses act in the standard manner of this heuristic, by taking the action with the most visitations: $\argmax_a M(s_1,a). $
The aggregated GumbelMCTS is implemented as taking the action which maximizes the \textit{aggregated} improved policy (Appendix \ref{app:baselines_implementation}). 
See Figure~\ref{fig:action_selection_ablation} for the effect of action selection for each agent.

In the learning experiment, GumbelMCTS picks an action stochastically from the improved policy in the first 8 moves, and after that picks deterministically the action recommended by SH. 
Similarly, in the learning experiment PMCTS selects by sampling from $\pi_{search}(s)$ at the first 8 actions, and then deterministically taking the $ \argmax_a \bar \pi(a|s) $ (see Appendix \ref{app:action_selection_and_improvement_survey}).


\subsection{Evaluation Metrics}
\label{app:eval_metrics}
In the standard RL environments, we use the observed averaged return across multiple independent evaluation episodes.
In the board game environments, we use BayesElo \cite{bayeselo}, a commonly used metric for comparing the relative strength of agents in two-player zero-sum games. Elo rating difference is related to the win-rate between agents. 
Bayes Elo finds the maximum likelihood ratings using MM \cite{minorization-maximization} where the prior on the Elo distributions is assumed to be uniform.
Elo is shift agnostic:
only the relative Elo between agents matters, not the absolute score.
To compute the average Elo across board games we shift a single agent to 0 in every environment (the $N=1$ Virtual Means), shift the Elo scores of all other agents in each environment accordingly, and then take the average across environments per agent.
For the policy improvement experiment in Go, we plot the win-rate against the PGX baseline agent.

We configure BayesElo with first-player advantage fixed at zero, since whenever we simulate a match between two agents, we play the same opening from both sides. Reported intervals are BayesElo's 95\% confidence intervals.



\subsection{Discrete vs. Continuous Action Spaces}
\label{app:continuous_action_mcts}
To extend PMCTS and MCTS to continuous actions we follow the popular approach of SampledMZ \citep{sampled_mz}, which samples $C$ actions from $\pi_\theta$ at each node in the search tree and treats $\{a_1,\dots,a_C\}$ as a discrete action space with uniform logits.

\subsection{Runtime Experiments}
\label{app:runtime_experiment}
We measured the runtime for every algorithm at every particle budget (powers of 2 from $1$ to $128$) with a simulation budget of 128. Each measurement consisted of 3 warmup runs to remove the overhead of JIT compilation, followed by 10 trials of 10 runs each. Every run used a different input state (generated ahead of time) to reflect real use. Each algorithm was tested using the same hardware (one NVIDIA A100 GPU), neural network, dependency versions, drivers, sequence of input states, etc. We plot the mean across trials, in effect the mean across all 100 runs, and the 95\% confidence intervals computed from the trials.


\subsection{Policy Improvement Experiment}
\label{app:training_experiment}
We use the baseline implementation of an on-policy AlphaZero training loop for the policy improvement experiments on board games by \citet{koyamada2023pgx}.
The training hyperparameters (learning rate, batch sizes, etc.) were tuned for the baseline GumbelMCTS based AlphaZero, and were not modified and thus are likely to favor the baseline agent (purple line in Figure \ref{fig:runtime_scaling}, right).

\subsection{LLM inference scaling on Sokoban}
\label{app:llm_experiment}
In this experiment, the LLM is playing instances of the game of Sokoban in which we use text-based rendering (ASCII).
Each Sokoban instance is a procedurally generated solvable 1-box puzzle. 
We sample 200 instances which we verify with breadth-first search are solvable within a 20-step window.
The model we used was Qwen2.5-14B-Instruct.
The LLM supplies the prior policy over the actions to search at each step.
The search action space contains five actions: the four moves plus a No-Op. 
Policy priors are formed from 8 independent completions per node at temperature 1.25, parsed strictly from $<$action$>$...$<$/action$>$; completions that fail to parse are attributed to the No-Op action, so malformed output becomes explicit prior mass rather than being silently discarded. 
A No-Op consumes an environment step.
For the value approximation at the leaves uniform random rollouts are used with a true model of the environment.

All methods were evaluated on the same 200 puzzle instances. 
Each
environment action used \(M=8\) parallel search iterations and \(N\in\{1,2,4,8,16,32\}\) particles. 
Each newly evaluated leaf used $8$ independent random rollout(s), each continuing until termination.

The system prompts used were:
\begin{lstlisting}
    You are solving a Sokoban puzzle. You control the player (P) on a grid.
    Your goal: push all boxes (X) onto the targets (O).
    
    Symbols:
      \# wall, \_ floor, O target, P player
      X box, V box on target, S player on target.
    
    Rules:
    - Move up, down, left, or right.
    - Push a box by moving into it (space behind must be empty).
    - You cannot pull boxes or push two at once.
    - Boxes cannot be pulled back. 

    A box pushed into a corner or against a wall may be stuck forever - plan multiple steps ahead before pushing.
    
    First reason about which moves are useful inside <think>...</think> tags.
    Then choose exactly one action inside <action>...</action> tags.
    Example: <think>The box is above me and the target is above the box. Pushing up will place it on the target without blocking anything.</think><action>up</action>
    
    Then in each step, the system prompt is concatenated with this message:
    
    Step \{step\_count\} of \{max\_steps\}.
    
    \{board\_text\}
    
    Moves so far: \{action, action, ...\} (only present if move\_history is non-empty).
    
    Your move?
\end{lstlisting}

The system prompt additionally contains three worked examples, each showing a small board and the corresponding $<$think$>$/$<$action$>$ response; these are given verbatim in the supplementary code.
Each prompt contains the current board and the list of actions taken so far, but not the previous boards or the model's earlier reasoning traces; carrying the full interaction history would substantially increase runtime.

%% file: appendix/hyperparameters.tex
\section{Hyperparameters}
\label{app:hyperparams}

\subsection{Hyperparameter tuning}
\label{app:hyperparams_tuning}
All hyperparameters except the temperature parameter $\eta$ introduced in this work follow \cite{gumbel_mz} as published in \cite{mctx}.
Search over $\eta$ is presented in Figure \ref{fig:ablations}, where $\eta = 1.5 $ dominates all other constant temperatures.
We leave the investigation of more complex temperature schedules depending on state, policy entropy, number of particles at the state, etc. to future work.

\subsection{Network architectures}
\label{app:net_archs}
Different architectures are used in different environments.
In the discrete-action and continuous-action single-agent environments, we use the architectures of \cite{trtpi}.
These are MLPs in all environments except Snake, where a CNN followed by an MLP is used.
See \cite{trtpi} for a full specification.



In the board game experiments we use the PGX \cite{koyamada2023pgx} baseline models for 9x9 Go and Gardner chess: small AlphaZero-style ResNets with 6 layers and 128 filters. For chess and 19x19 Go where we trained our own models, we used the same architecture as the PGX models, except for 19x19 Go we used 16 residual blocks instead of 6.

\subsection{Hyperparameters}
This subsection covers the hyperparameters used for the experiments. \autoref{tab:hps_shared} gives the hyperparameters shared among experiments: These are the PUCT parameters used by Virtual Means and Virtual Losses, and the Q-value scaling used by both PMCTS and Gumbel MCTS. \autoref{tab:hps_board_games} lists hyperparameters specific to the experiments in 9x9 Go and Gardner Chess. \autoref{tab:hps_cont_disc} details the episode length per environment used. The sampled actions refer to how many actions we sample from the continuous action space for search. Finally, \autoref{tab:hps_policy_opt} gives hyperparameters for the policy optimization experiment.

\begin{table}[h]\centering
\caption{Shared Hyperparameters for All Experiments}
\begin{tabular}{lll}
\toprule
Category & Hyperparameter & Value \\
\midrule
PUCT (used by Virtual Means / Losses) & $c_{\text{base}}$ & 19652 \\
                                      & $c_{\text{init}}$ & 1.25 \\
\midrule
Q-value Scaling for the Improved Policy & $\beta_i$ & $(c_{\text{visit}} + \max_a M_i(s,a))\, c_{\text{scale}}$ \\
                                        & $c_{\text{visit}}$ & 50 \\
                                        & $c_{\text{scale}}$ & 0.1 \\
\bottomrule
\label{tab:hps_shared}
\end{tabular}
\end{table}

\begin{table}[H]
\centering
\caption{Hyperparameters for Board Game Experiments}
\begin{tabular}{@{}llc@{}}
\toprule
\textbf{Category} & \textbf{Hyperparameter} & \textbf{Value} \\ \midrule
Gumbel MCTS
 & Sampled Actions & 16 \\
 & Gumbel Scale & 0.0 \\ \midrule
PMCTS
 & Temperature $\eta$ & 1.5 \\ \midrule
Model
 & Model for 9x9 Go & PGX baseline v0 \\
 & Model for Gardner Chess & PGX baseline v0 \\ \midrule
Opening Book
 & Size per Environment & 20000 \\
 & Allowed Value Range & $[-0.3, 0.3]$ \\
 & Steps for 9x9 Go & 12 \\
 & Steps for 19x19 Go & 12 \\ 
 & Steps for Gardner Chess & 8 \\ 
 & Steps for Chess & 8 \\ \bottomrule
\end{tabular}
\label{tab:hps_board_games}
\end{table}

\begin{table}[h]\centering
\caption{Training hyperparameters for the chess and 19x19 Go evaluation networks.}
\begin{tabular}{lll}
\toprule
Hyperparameter & Chess & 19x19 Go \\
\midrule
Search used in self-play        & GumbelMCTS & GumbelMCTS \\
Simulation budget $M$           & 24         & 32 \\ 
Considered actions              & 8          & 16 \\
Discount                        & 0.997      & 0.997 \\
Self-play steps per iteration   & 256        & 512 \\
Self-play batch size            & 1024       & 1024 \\ 
Training batch size             & 256        & 512 \\ 
Learning rate                   & $10^{-3}$  & $10^{-4}$ \\
Momentum                        & -          & 0.9 \\
Training iterations             & 250        & 670 \\
\bottomrule
\end{tabular}
\end{table}

\begin{table}[H]
\centering
\caption{Hyperparameters for Continuous and Discrete Environments}
\begin{tabular}{@{}llcc@{}}
\toprule
\textbf{Environment} & \textbf{Type} & \textbf{Episode Length} & \textbf{Num. Sampled Actions} \\ \midrule
Snake &	Discrete (jumanji) & 4000 &	- \\
RubiksCube & Multi-Discrete (jumanji) & 20 & - \\
Ant	& Continuous (brax) & 1000  & 30 \\
HalfCheetah	& Continuous (brax) & 1000 & 30 \\
Humanoid & Continuous (brax) & 1000 & 30 \\ \bottomrule
\end{tabular}
\label{tab:hps_cont_disc}
\end{table}

\begin{table}[H]
\centering
\caption{Hyperparameters for Policy Optimization Experiment (AlphaZero Training)}
\begin{tabular}{@{}llc@{}}
\toprule
\textbf{Category} & \textbf{Hyperparameter} & \textbf{Value} \\ \midrule
Network
 & Architecture & ResNet V2 \\
 & Number of Filters & 128 \\
 & Number of Residual Blocks & 6 \\ \midrule
Self-Play
 & Batch Size & 1024 \\
 & Simulations per Step & 24 \\
 & Max Number of Steps & 256 \\
 & Random Plies & 8 \\ \midrule
Gumbel MCTS
 & Sampled Actions & 8 \\
 & Gumbel Scale & 1.0 \\ \midrule
PMCTS
 & Number of Particles & 4 \\
 & Temperature $\eta$ & 1.5 \\ \midrule
Training
 & Batch Size & 256 \\
 & Discount & $0.997$ \\
 & Optimizer & Adam \\
 & Learning Rate & $0.001$ \\ \bottomrule
\end{tabular}
\label{tab:hps_policy_opt}
\end{table}